\documentclass[letterpaper]{article} % DO NOT CHANGE THIS
\usepackage{aaai2026_arxiv}  % DO NOT CHANGE THIS
\usepackage{times}  % DO NOT CHANGE THIS
\usepackage{helvet}  % DO NOT CHANGE THIS
\usepackage{courier}  % DO NOT CHANGE THIS
\usepackage[hyphens]{url}  % DO NOT CHANGE THIS
\usepackage{graphicx} % DO NOT CHANGE THIS
\usepackage{natbib}  % DO NOT CHANGE THIS AND DO NOT ADD ANY OPTIONS TO IT
\usepackage{caption} % DO NOT CHANGE THIS AND DO NOT ADD ANY OPTIONS TO IT
\usepackage{algorithmic,amsthm}
\usepackage[linesnumbered,ruled]{algorithm2e}

\usepackage{newfloat}
\usepackage{listings}
\DeclareCaptionStyle{ruled}{labelfont=normalfont,labelsep=colon,strut=off} % DO NOT CHANGE THIS
\usepackage{pifont}
\usepackage{etoolbox}
\usepackage[table, xcdraw]{xcolor}
\usepackage{multirow}      % for \multirow
\usepackage{graphicx}      % for \rotatebox
\usepackage{array}         % nicer column

\usepackage{mathtools}
\usepackage{amsmath}
\usepackage{amssymb}
\usepackage{makecell}
\usepackage{amsthm}
\usepackage{dsfont}
\usepackage{enumitem}
\newtheorem{theorem}{Theorem}[section]
\newtheorem{corollary}[theorem]{Corollary}

\newtheorem{claim}[theorem]{Claim}
\newtheorem{lemma}[theorem]{Lemma}

\newtheorem{remark}[theorem]{Remark}
\newcommand{\ones}{\mathds{1}}

\newcommand{\Rm}{R_\textnormal{max}}
\newcommand{\cA}{\mathcal{A}}

\newcommand{\cC}{\mathcal{C}}

\newcommand{\cF}{\mathcal{F}}

\newcommand{\cM}{\mathcal{M}}

\newcommand{\KG}{K_\Gamma}

\newcommand{\p}[1]{\ifblank{#1}{\upsilon}{\upsilon(#1)}}
\newcommand{\inter}[1]{\ifblank{#1}{\inter}{\textnormal{int}(#1)}}

\newcommand{\bE}{\mathbb{E}}
\newcommand{\bI}{\mathbb{I}}

\newcommand{\bP}{\mathbb{P}}
\newcommand{\bR}{\mathbb{R}}

\newcommand{\cT}{\mathcal{T}}

\newcommand{\SA}{|\cS| |\cA|}
\newcommand{\piS}{\pi^*}

\newcommand{\bAinit}{\bA^\textnormal{init}}
\newcommand{\bAtrans}{\bA^\textnormal{trans}}
\newcommand{\bAnoise}{\bA^\textnormal{noise}}
\newcommand{\bAnlin}{\bA^\textnormal{nonlin}}

\newcommand{\hDel}{\hat{\Delta}}
\newcommand{\bDeltrans}{\bDelta^{\textnormal{trans}}}
\newcommand{\bDelinit}{\bDelta^{\textnormal{init}}}
\newcommand{\bDelnoise}{\bDelta^{\textnormal{noise}}}
\newcommand{\bDelnlin}{\bDelta^{\textnormal{nonlin}}}
\newcommand{\hDelnoise}{\hat{\Delta}^\textnormal{noise}}

\newcommand{\bA}{\textbf{A}}

\newcommand{\cP}{\mathcal{P}}
\newcommand{\hcP}{\hat{\mathcal{P}}}
\newcommand{\cR}{\mathcal{R}}
\newcommand{\cS}{\mathcal{S}}
\newcommand{\cY}{\mathcal{Y}}

\newcommand{\taut}{\tau_{t}}
\newcommand{\tauT}{\tau_{T}}
\newcommand{\tauk}{\tau_{k}}

\newcommand{\bDelta}{\bar{\Delta}}

\newcommand{\at}{\alpha_t}
\newcommand{\ak}{\alpha_k}

\newcommand{\bQ}{\bar{Q}}

\newcommand{\ba}{\bar{a}}

\newcommand{\QS}{Q^*}

\newcommand{\pQ}{\pi_Q}

\newcommand{\tstr}{t_*}
\newcommand{\eS}{e^*}
\newcommand{\bht}{\hat{b}_t}
\newcommand{\bhk}{\hat{b}_k}

\newcommand{\Aht}{\hat{A}_t}

\newcommand{\Ahk}{\hat{A}_k}

\newcommand{\yti}{y^i_t}
\newcommand{\yt}{y_t}

\newcommand{\CD}{C_\Delta}

\newcommand{\CA}{C_\textnormal{A}}
\newcommand{\Cb}{C_\textnormal{b}}
\newcommand{\Cgm}{C_\Gamma}

\newcommand{\spf}{\textnormal{sp}}

\newcommand{\CE}{C_\textnormal{E}}

\newcommand{\Cl}{C^{(1)}}
\newcommand{\Cq}{C^{(2)}}

\newcommand{\Cah}{H_\textnormal{Avg}}

\DeclareMathOperator*{\argmax}{arg\,max}
\DeclareMathOperator*{\argmin}{arg\,min}

\newcommand{\TV}{\textnormal{TV}}
\newcommand{\ep}{\varepsilon_p}
\newcommand{\er}{\varepsilon_r}

\newcommand{\cmark}{\ding{51}} % Checkmark symbol
\newcommand{\xmark}{\ding{55}}

\newcommand{\floor}[1]{\left\lfloor #1 \right\rfloor}

\newcommand{\cstr}{C_*}

\title{Parameter-free Optimal Rates for Nonlinear Semi-Norm Contractions with Applications to $Q$-Learning}
\author {
    Ankur Naskar\textsuperscript{\rm 1},
    Gugan Thoppe\textsuperscript{\rm 1},
    Vijay Gupta\textsuperscript{\rm 2}
}
\affiliations {
    \textsuperscript{\rm 1}Computer Science and Automation, Indian Institute of Science, Bengaluru\\
    \textsuperscript{\rm 2}Electrical and Computer Engineering, Purdue University, IN, USA\\
    ankurnaskar@iisc.ac.in, gthoppe@iisc.ac.in, gupta869@purdue.edu
}

\usepackage{bibentry}
\begin{document}

\maketitle

\begin{abstract}
Algorithms for solving \textit{nonlinear} fixed-point equations---such as average-reward \textit{$Q$-learning} and \textit{TD-learning}---often involve semi-norm contractions. Achieving parameter-free optimal convergence rates for these methods via Polyak–Ruppert averaging has remained elusive, largely due to the non-monotonicity of such semi-norms. We close this gap by (i.) recasting the averaged error as a linear recursion involving a nonlinear perturbation, and (ii.) taming the nonlinearity by coupling the semi-norm's contraction with the monotonicity of a suitably induced norm. Our main result yields the first parameter-free $\tilde{O}(1/\sqrt{t})$ optimal rates for $Q$-learning in both average-reward and exponentially discounted settings, where 
$t$ denotes the iteration index. The result applies within a broad framework that accommodates synchronous and asynchronous updates, single-agent and distributed deployments, and data streams obtained either from simulators or along Markovian trajectories.
\end{abstract}

% % %     \link{Code}{https://aaai.org/example/code}
% % %     \link{Datasets}{https://aaai.org/example/datasets}
%     \link{Extended version}{https://arxiv.org/abs/2508.05984}

\section{Introduction}
\label{s: introduction}
Stochastic fixed-point iterations with contractive operators \cite{borkar2009} permeate many fields. In Reinforcement Learning (RL), Temporal Difference (TD)\footnote{Although $Q$-learning falls within the broader family of TD methods, we use TD exclusively for policy-evaluation algorithms.} learning and $Q$-learning are canonical examples \cite{BertsekasTsitsiklis1996,sutton2018reinforcement,szepesvari2022algorithms}. Game-theoretic approaches to decentralized learning and Nash-equilibrium computation adopt the same template \cite{sayin2021decentralized,hu2003nash}. In optimization, stochastic fixed-point iterations underpin methods for regularized least squares in over-parameterized models and for low-rank matrix completion \cite{li2018algorithmic,marjanovic2012l_q}. They are likewise of growing importance in nonlinear dynamical systems, where fixed-point theory has become an important tool for analysis and control design, particularly in robotics~\cite{lohmiller1997applications,manchester2014output,manchester2017unifying,tsukamoto2021contraction}.

Here, we focus on stochastic fixed-point iterations driven by \textit{semi-norm} contractions. Their generic update rule is 
\begin{equation}
\label{eq:Generic.SA.fixed.point.iteration}
    Q_{t+1}  = Q_t + \at\left[h(Q_t, \yt) - Q_t \right],
\end{equation}
where $(\yt)$ is an ergodic Markov chain on a state space $\cY$ with a unique stationary distribution $\eta$ and $h: \bR^d \times \cY \to \bR^d$ is such that the expected map $H(Q) = \bE_{y \sim \eta} h(Q, y)$ is a semi-norm contraction. In other words, we have a semi-norm $\p{}: \bR^d \to \bR$ and a factor $\beta \in [0, 1)$ so that 
\begin{equation}
\label{e:semi.norm.contraction}
    \p{H(x) - H(y)} \leq \beta\ \p{x - y} \quad \forall x, y \in \bR^d.
\end{equation}
A semi-norm is similar to a norm that need not satisfy the positive definiteness condition. Specifically, for all $x, y \in \bR^d$ and $\lambda \in \bR,$ a semi-norm $\p{}$ satisfies non-negativity ($\p{x} \geq 0$), absolute homogeneity ($\p{\lambda x} = |\lambda| \p{x}$), and sub-additivity ($\p{x + y} \leq \p{x} + \p{y}$), but not necessarily positive definiteness ($\p{x} = 0 \nRightarrow x = 0$). As a result, every norm is a semi-norm, but the converse does not hold. 

%%%%%%%%%%%%%%%%%%%%%%%%%%%%%%%%%%%%%%%%%%%%%%%%%%%%%%%%%%%%%%%%%%%%%%%%%%%%%%%%%%%%%%%%%%%%%%%%%%%%%%%%
\begin{table*}[ht]
\centering
\begin{tabular}{!{\vrule width0.6pt}c!{\vrule width0.6pt}p{16.75em}|c|c|c|c|c!{\vrule width0.6pt}}
\Xhline{0.6pt}
& \textbf{Reference} & \textbf{Distributed} &  \textbf{Discounting} &  \textbf{ Asynchronous} & \textbf{\makecell{Optimal \\ rate}}  & \textbf{\makecell{Universal \\ stepsize}} 
\\
\Xhline{0.6pt}
\multirow{15}{0.4cm}
{
      \rotatebox[origin=c]{90}{\parbox{2.5cm}{\centering TD Learning}}
}
% \hline
%
& \citet{dalal2018td0} & \xmark & Exp & \cmark & \xmark & \cmark 
\\
\cline{2-7}
&  \citet{patil2023finite} & \xmark & Exp & \cmark & \cmark & \cmark
\\
% \cline{2-7}
% &  \citet{chen2021finite} & \xmark & Exp & \cmark & \cmark & \xmark 
% \\
\cline{2-7}
&  \citet{lakshminarayanan2018linear} & \xmark & Exp and Avg & \cmark & \cmark &\xmark
\\
\cline{2-7}
&  \citet{bhandari2018finite} & \xmark & Exp & \cmark & \cmark  & \xmark  
 \\
\cline{2-7}
& \citet{chen2021lyapunov} & \xmark & Exp & \cmark & \cmark & \xmark
\\
\cline{2-7}
& \citet{durmus2025finite} & \xmark & Exp & \cmark & \cmark & \xmark
\\
\cline{2-7}
& \citet{repec:inm:oropre:v:69:y:2021:i:3:p:950-973} & \xmark & Exp & \cmark & \cmark & \xmark 
\\
\cline{2-7}
& \citet{chen2025non} & \xmark & Exp and Avg & \cmark & \cmark & \xmark
\\
\cline{2-7}
&  \citet{dal2023federated} & \cmark & Exp & \cmark & \cmark & \xmark
\\
\cline{2-7}
& \citet{khodadadian22a} & \cmark & Exp &  \cmark & \cmark & \xmark 
\\
\cline{2-7}
&  \citet{liu2023distributed} & \cmark & Exp & \cmark & \cmark & \xmark
\\
\cline{2-7}
& \citet{wang2024federated} & \cmark & Exp & \cmark & \cmark & \xmark
\\
\cline{2-7}
& \citet{naskar2024federated} & \cmark & Exp and Avg  & \cmark  & \cmark & \cmark
\\
\Xhline{0.6pt}
\multirow{6}{0.4cm}
{
      \rotatebox[origin=c]{90}{\parbox{2.5cm}{\centering $Q$-Learning}}
}
& \citet{even2003learning} & \xmark  & Exp & \xmark & \xmark & \cmark 
\\
\cline{2-7}
&  \citet{wainwright2019stochastic} & \xmark & Exp & \xmark & \cmark &  \xmark
\\
\cline{2-7}
& \citet{zhang2021finite} & \xmark & Avg & \xmark & \cmark & \xmark
\\
\cline{2-7}
& \cite{chen2021lyapunov} & \xmark & Exp & \cmark & \cmark & \xmark 
\\
\cline{2-7}
& \citet{li2023statistical} & \xmark & Exp & \xmark & \cmark & \cmark
\\
\cline{2-7}
& \multirow{2}{*}{\textbf{Our Work}} & \cmark & Avg & \xmark & \cmark & \cmark
\\
\cline{3-7}
&  & \cmark & Exp & \cmark & \cmark & \cmark
\\
\Xhline{0.6pt}
\end{tabular}
\caption{\label{tab:lit.survey} Comparison of our work with the existing literature on TD-learning and $Q$-learning algorithms. In the column labeled discounting, Exp refers to exponential, while Avg refers to average reward. }
\end{table*}
%%%%%%%%%%%%%%%%%%%%%%%%%%%%%%%%%%%%%%%%%%%%%%%%%%%%%%%%%%%%%%%%%%%%%%%%%%%%%%%%%%%%%%%%%%%%%%%%%%%%%%%

The role of semi-norm contractions in RL methods such as TD and $Q$-learning is fundamental. In these methods, the Bellman operator is contractive---under a semi-norm in the average-reward formulation and under a true norm (hence, also a semi-norm) in exponential discounting \cite{BertsekasTsitsiklis1996}. The generic recursion in \eqref{eq:Generic.SA.fixed.point.iteration} subsumes all variants of TD methods and $Q$-learning---including the ones with synchronous and asynchronous updates, single-agent and distributed deployments, and data drawn from generative models and Markovian trajectories.

Achieving optimal convergence rates in such algorithms is of fundamental interest. Several works have obtained such bounds for TD algorithms and $Q$-learning in their various forms. TD methods, for instance, have been looked at in  \cite{dalal2018td0, lakshminarayanan2018linear,bhandari2018finite,zhang2021finite, liu2023distributed,khodadadian22a,dal2023federated} and \cite{wang2024federated}. On the other hand, tabular $Q$-learning has been the focus of \cite{even2003learning,chen2021finite,zhang2021finite} and \cite{chen2025non}. Table~\ref{tab:lit.survey} summarizes the key differences in these results.

However, a key drawback of these results is that the optimal expected error bound, $\tilde{O}(1/\sqrt{t})$, is achieved only with a stepsize $\alpha_t = c/t$, where the constant $c$ depends on unknown transition probabilities. Such a $c$ is rarely available, making these rates effectively impractical.

For \textit{linear} stochastic approximation, \cite{polyak1992acceleration} and \cite{ruppert1991stochastic}  proposed a remarkable solution---now called Polyak–Ruppert averaging. Applied to an update rule like \eqref{eq:Generic.SA.fixed.point.iteration}, this procedure works in two steps:
\begin{enumerate}
    \item \textbf{Generate base iterates}: Run $(Q_t)$ with the parameter-free stepsize $\alpha_t = 1/(t + 1)^\alpha$ for some $\alpha \in (1/2, 1).$   

    \item \textbf{Average them}: Form the running mean $\bQ_T = \frac{1}{T}\sum_{t=0}^{T-1} Q_t,$ again without any problem-specific tuning. 
\end{enumerate}
While this procedure yields a suboptimal convergence rate of $\tilde{O}(1/t^{\alpha/2})$ for the $(Q_t)$ sequence, the averaged sequence $(\bQ_t)$
attains the desired optimal rate of $\tilde{O}(1/\sqrt{t}).$

% Given the versatility of \eqref{eq:Generic.SA.fixed.point.iteration}, understanding when and how fast it converges is of fundamental importance~\cite{nevel1976stochastic,haque2023tight}. As part of analysis of many RL algorithms, optimal non-asymptotic convergence rates of this iteration has accordingly received much attention in the literature~\cite{tsitsiklis1994asynchronous,beck2012error,bhandari2018finite}. {\color{red} Should give some examples here}

For TD-learning with linear function approximation---a special case of linear stochastic approximation---Polyak–Ruppert averaging already achieves parameter-free optimal rates, both for exponential discounting~\cite{patil2023finite} and for average reward~\cite{naskar2024federated}.

$Q$-learning with Polyak--Ruppert averaging is challenging due to the algorithm’s inherent nonlinearity. A recent breakthrough by \citet{li2023statistical} addresses this challenge for {\em synchronous Q-learning} with \textit{exponential discounting}. In this setting, the Bellman operator is a contraction in the monotone $\|\cdot\|_\infty$ norm. The authors’ key idea is to construct two auxiliary sequences, $(L_t)$ and $(U_t)$, whose Polyak--Ruppert averages bound the error in $(\bar{Q}_t)$ from below and above, respectively. Each auxiliary sequence follows the template of a linear stochastic approximation with a rapidly vanishing nonlinear remainder. Classical Polyak--Ruppert analysis then yields a $\tilde{O}(1/\sqrt{t})$ rate for both. By the monotonicity of $\|\cdot\|_\infty$, these bounds transfer directly to $\bar{Q}_t$.

For synchronous $Q$-learning with exponential discounting, the raw iterates \( (Q_t) \) already achieve the optimal rate for the stepsize $\alpha_t = \frac{1}{1 + (1 - \gamma)t},$ where \(\gamma\) is the (known) discount factor \cite[Corollary~3]{wainwright2019stochastic}. Polyak-Ruppert averaging, therefore, is unnecessary in this setting if the goal is only to get parameter-free optimal rates.

% For exponential discounting (resp. average reward), asynchronous $Q$-learning (resp. both synchronous and asynchronous $Q$-learning) has an effective contraction factor that depends on the transition probabilities. Hence, achieving parameter-free optimal rates in these cases is open.

For asynchronous $Q$-learning in exponential discounting---and for both synchronous and asynchronous versions under average reward---the contraction factor depends on unknown transition probabilities. Hence, parameter-free optimal rates in these cases remain open.

While the ideas of \citet{li2023statistical} extend to asynchronous Q‐learning in exponential discounting (with some effort), they break down in the average-reward case. The challenge stems from the non-monotonicity of the span semi-norm
\begin{equation}
\label{e:span.semi.norm}
    \|x\|_{\spf} := \max_i x(i) - \min_i x(i),
\end{equation}
under which the average-reward Bellman operator is contractive. Specifically, $0 \leq x \leq y$ does not imply $\|x\|_{\spf} \leq \|y\|_{\spf}$. As a result, although auxiliary lower- and upper-bounding sequences can still be constructed and shown to achieve parameter-free optimal rates, $\|\cdot\|_{\spf}$'s non-monotonicity prevents these rates from transferring to $(\bar{Q}_t)$.

Our work’s \textbf{key highlights} are as follows:
\begin{enumerate}
    \item We are the first to show that Polyak--Ruppert averaging achieves \textbf{parameter-free optimal rates} even for semi-norm contractive nonlinear fixed-point iterations.

    \item Our result \textbf{applies directly to $Q$-learning in both the average-reward and exponentially discounted settings}, providing the first parameter-free optimal expected error rates of $\tilde{O}(1/\sqrt{t})$.

    \item Our proof is novel and may be of independent interest. Specifically, our proof proceeds in two main steps: (i) decomposing the error in $\bar{Q}_t$ into a linear recursion with a nonlinear perturbation (following \cite{li2023statistical}); and (ii) showing rapid decay of the nonlinear term by exploiting the \textbf{contraction of the semi-norm} together with the \textbf{equivalence between its induced norm and a suitably chosen monotone norm}. 
    
    % Our approach thereby enables the Polyak--Ruppert analysis to go through.
\end{enumerate}

\section{Problem Setup and Our Goal}
\label{s: setup.main}
Although our contributions are novel even in single-agent settings, for generality, we pose our problem in a distributed architecture with a central server and $N$ agents. Each agent $i\in[N]:=\{1,\ldots,N\}$ is associated with a function $h_i:\bR^d\times\cY\to\bR^d$ and a $\cY$-valued stochastic process $(y_t^i)$. At time $t\in\{0,1,\ldots\}$, upon receiving a query $Q\in\bR^d$, agent $i$ returns the sample $h_i(Q,y_t^i)$.  The processes $\{(y_t^i)\}_{i\in[N]}$ are independent across agents and independent of the iterate sequence $(Q_t)$. For each $i$, $(y_t^i)_{t\ge0}$ evolves as a time-homogeneous Markov kernel with stationary distribution $\eta_i$. The corresponding mean operator $H_i:\bR^d\to\bR^d$ is 
\begin{equation}
\label{e:H_i.defn}
    H_i(Q)=\bE_{y\sim\eta_i}h_i(Q,y).
\end{equation}
We assume that $H_i$ is a $\beta_i$-contraction with respect to the common semi-norm $\p{}:\bR^d\to\bR$; see \eqref{e:semi.norm.contraction}.

The aim of this setup is to find a fixed point $\QS\in \bR^d$ of the averaged map $H: \bR^d \to \bR^d$ given by 
\begin{equation}
\label{e:average.map.defn}
    H(Q) = \frac{1}{N} \sum_{i} H_i(Q).    
\end{equation}
That is, a vector $\QS$ such that 
\begin{equation}\label{e:semi.norm.fixed.point}
    H(\QS) - \QS \in E,
\end{equation}
where 
\begin{equation}
\label{e:norm.vanishing.subspace}
    E = \{x \in \bR^d: \p{x} = 0\}
\end{equation}
is the linear subspace on which $\p{}$ vanishes. Such a fixed point exists (and is unique modulo $E$) since $H: \bR^d \to \bR^d$ is itself a semi-norm contraction with a contraction factor 
\begin{equation}
\label{e:H.semi-norm.contraction.factor}
    \beta := \frac{1}{N} \sum_{i = 1}^N \beta_i \leq \max\{\beta_1, \ldots, \beta_N\}.
\end{equation}

A common way to obtain such a $\QS$ is for the agents to execute the distributed stochastic fixed-point iteration 
\begin{equation}\label{e: SA.main.update}
    Q_{t+1} = Q_t + \at\bigg[\frac{1}{N}\sum_{i=1}^{N} h_i(Q_t, \yti) - Q_t \bigg], \quad t \geq 0.
\end{equation}
%
% Our goal is to derive optimal convergence rates for such methods in a \textit{parameter-free} manner, i.e., without needing the stepsize $\alpha_t$ to depend on unknown problem parameters.
%
% Within RL, there are many key algorithms that follow the template in \eqref{e: SA.main.update}, including $TD(0)$, $TD(\lambda)$, and $Q$-learning under both discounted and average-reward formulations. This statement is true for all their practical variants---such as those involving single-agent and distributed deployments, synchronous and asynchronous updates, and data drawn from generative and Markovian trajectories. All these variants are unified through the Bellman operators that serve as the contraction maps: in the average-reward case, the contraction is in the span semi-norm $\|\cdot\|_{\spf},$ whereas in the discounted case, it is in the $\|\cdot\|_\infty$-norm. 
%
% For linear $TD$-style updates, Polyak–Ruppert averaging already yields parameter-free optimal rates \cite{patil2023finite,naskar2024federated}. For non-linear methods---notably $Q$-learning---achieving similar rates is an open problem, even in the single-agent setting. Thus, 
%
Our goal then is to determine whether the optimal $\tilde{O}(1/\sqrt{t})$ rate, with problem-independent stepsizes, holds for Polyak-Ruppert  average sequence $(\bQ_t)$ given by
\begin{equation}\label{e : SA.PR.update}
    \bQ_{t+1} = \bQ_t + \frac{1}{t+1}\left[ Q_t - \bQ_t \right], \quad t \geq 0.
\end{equation}

Since our framework allows for semi-norm contractions, a positive result would instantly provide parameter-free optimal rates to all avatars of TD learning and $Q$-learning.

\section{Main Result and RL Applications}
\label{s: results.main}

In Section~\ref{s: results.SA}, we state our assumptions and present our main convergence-rate theorem for the Polyak–Ruppert average \eqref{e : SA.PR.update} of the stochastic fixed-point iteration in \eqref{e: SA.main.update}. Subsequently, in Section~\ref{s: application.RL}, as an illustration, we show how this result applies to $Q$-learning under both average reward and exponential discounting, yielding the first parameter-free optimal convergence rates for these methods.

\subsection{Parameter-Free Optimal Convergence Rates for the Stochastic Fixed-Point Iteration \eqref{e: SA.main.update}}
\label{s: results.SA}

We begin with the following fact about semi-norms.

% For a subspace $ E$, let $E^\perp$ denote its orthogonal complement. 
%
\begin{lemma}\!\textnormal{\cite[Proposition 2.1]{chen2025non}}
\label{lem:semi-norm.norm.coupling}
    For any semi-norm $\p{}: \bR^d \to \bR,$ one can define an induced norm $\|\cdot\|$ such that, for all $Q \in \bR^d,$
    \begin{equation}\label{e: induced}
        \p{Q} = \min_{e\in E} \|Q - e\|,
    \end{equation}
    where $E$ is as defined in \eqref{e:norm.vanishing.subspace}.
\end{lemma}

All norms are equivalent in finite dimensions. So, for any monotone norm\footnote{For all $x,y\in \bR^d,$ if $0\leq x\leq y$ then $\|x\|_m \leq \|y\|_m.$} $\|\cdot\|_m$, there exist $c_\ell, c_u > 0$ so that 
\begin{equation}\label{e: span.semi.norm.monotone.induced.norm}
    c_\ell\|Q\|_m \leq \|Q\|\leq c_u\|Q\|_m, \qquad \forall Q \in \cR^d.
\end{equation}

Next, we state our assumptions. For all $t \geq 0,$ let $(\cF_t)$ be the filtration sequence of $\sigma$-fields defined by
\begin{equation}
    \cF_t := \sigma\left(\{Q_0\} \cup \{ y_k^i: i\in [N], k<t \}\right).
\end{equation}
%
% Further, suppose that the following assumptions hold.
%
\begin{enumerate}[label=$\mathbf{\cA_\arabic{enumi}}$, leftmargin=*, align=left]
    \item \textbf{Local update functions:} For each agent $i\in[N],$ there exist $b_i:\bR^d\times \cY \to \bR^d$ and $A_i:\bR^d\times \cY \to \bR^{d\times d}$ and constants $\CA, \Cb, \cstr > 0$ such that $h_i(Q, y) = b_i(Q,y) + A_i(Q,y)Q$ and the following holds:
    \begin{enumerate}[leftmargin=*, align=left]        
        \item $A_i(Q,y)e = e$ for all $e\in E$ (see \eqref{e:norm.vanishing.subspace}). 
        
        \item $A_i(Q,y)Q \geq A_i(Q',y)Q$  for all $Q, Q' \in \bR^d$, where the inequality is in a coorindate sense.
        
        \item $\|A_i(Q,y)\| \leq \CA$ and $\|b_i(Q,y)\| \leq \Cb,$ where $\|\cdot\|$ is the induced norm (see \eqref{e: induced}).
        
        \item For any semi-norm fixed point $\QS$ satisfying \eqref{e:semi.norm.fixed.point}, $A_i(Q,y) = A_i(\QS,y)$ and $b_i(Q,y) = b_i(\QS,y)$ if $\p{Q-\QS}<\cstr.$
    \end{enumerate} \label{a: SA.update.function}
    
    \item \textbf{Noise:} There are constants $\CE>0$ and $\rho\in (0,1)$ such that, for agent $i \in [N]$ and time $\tau \leq t,$ the Markov chain $(y^i_t)$  satisfies $\|\bP(\yti \in  \cdot | \cF_{t-\tau}) - \eta_i (\cdot)\|_{\TV} \leq \CE\ \rho^\tau,$
    where $\|\cdot\|_{\TV}$ is the total variation distance. \label{a: SA.noise}
    %
    % \item \label{a: semi-norm.norm.relation}
    % % \textbf{Induced monotone norm:} The norm $\|\cdot\|$ induced by $\p{}$---as detailed in Lemma~\ref{lem:semi-norm.norm.coupling}---is monotone.
    
    \item \textbf{Raw iterate convergence:} There is a constant $C_Q>0$ such that, for all $t \geq \tstr,$ $\bE\p{Q_t - \QS}^2 \leq C_Q\ \taut \ \at,$ 
    where $\taut:= \min\{ \tau > 0: \rho^\tau \leq \at^2\}$ and $\tstr:= \min\{ t> 0: t\geq 2\taut \text{ and } \at<1/\kappa \}\footnote{$\kappa$ defined in Table~\ref{tab: constants} (appendix).}.$ \label{a: raw.iterate.convergence.rate}
\end{enumerate}

\begin{remark}
    Assumption \ref{a: SA.update.function} holds for any piece-wise linear update rule, including TD-learning and $Q$-learning in both exponential discounting and average-reward settings. At this point, it is unclear if \ref{a: SA.update.function} will hold beyond these cases.
\end{remark}

\begin{remark}
    Assumption \ref{a: SA.noise} is called the geometric mixing time assumption and is standard for algorithms using Markovian sampling \cite{durmus2025finite, zhang2021finite, chen2025non}. 
\end{remark}
%

% \begin{remark}
% \label{rem:span.semi.norm.monotone.induced.norm}
%     Assumption~\ref{a: semi-norm.norm.relation} is immediate whenever $\p{}$ is itself a monotone norm like $\|\cdot\|_p$ for $p \in \{1, \ldots, \} \cup \{\infty\}.$ It also holds for the span semi-norm $\|\cdot\|_{\spf},$ in which case the induced norm is $\|\cdot\|_\infty,$ which is clearly monotone.
% \end{remark}

\begin{remark}
    The condition in Assumption~\ref{a: raw.iterate.convergence.rate} has been derived for several variants of TD and Q-learning algorithms; see the references in Table~\ref{tab:lit.survey} for details.
\end{remark}

Our main result can now be stated as follows. 

\begin{theorem}\label{thm: SA.convergence}
    Consider the stochastic fixed-point iteration \eqref{e: SA.main.update} under Assumptions \ref{a: SA.update.function}--\ref{a: raw.iterate.convergence.rate}. Choose the stepsize $\at = \frac{1}{(t+1)^\alpha}$ for some $\alpha \in (1/2,1)$. Let $\taut$ and $\tstr$ be as in Assumption~\ref{a: raw.iterate.convergence.rate}, and let $\rho$ be as in  Assumption~\ref{a: SA.noise}. Then, for $T > \tstr,$ the Polyak-Ruppert averages  given in \eqref{e : SA.PR.update} satisfy
    \[
        \bE\p{\bQ_T - \QS} \leq \frac{\Cl\sqrt{\tauT}}{\sqrt{NT}} + \frac{\Cq\ln(T)}{T^{\alpha}},
    \]
    where constants $\Cl$ and $\Cq$ are as in Table~\ref{tab: constants} (appendix).
\end{theorem}

\begin{remark}[Optimal Convergence Rate]
    The factor $\tau_T$ grows only logarithmically with $T$, i.e., $\tau_T = O(\ln T)$. Separately, since $\alpha > \tfrac{1}{2}$, the second term is $o(1/T)$ and thus asymptotically negligible. Hence, $\bE\p{\bQ_T-\QS} = \tilde{O}(1/\sqrt{T}),$ which implies that this result gives the optimal expected error rate up to logarithmic factors.
\end{remark}

\begin{remark}[Linear Speedup]
    The leading term in our bound has a factor of $1/\sqrt{N},$ implying that the iteration complexity improves linearly with the number $N$ of agents.
\end{remark}

\begin{remark}[Parameter-Free Stepsize]
    The exponent $\alpha \in (1/2, 1)$ in the stepsize $\at = 1/(t+1)^\alpha$ can be fixed universally; it requires no knowledge of problem-specific parameters such as the contraction coefficient or the mixing-time.
\end{remark}
%

%%%%%%%%%%%%%%%%%%%%%%%%%%%%%%%%%%%%%%%%%%%%%%%%%%%%%%%%%%%%%%%%%%%%%%%%%%%%%%%%%%%%%%%%%%%%%%%%%%%%%%%%%%%%%%%%%%%%%%%%

\subsection{Reinforcement Learning Applications}
\label{s: application.RL}

We next use Theorem \ref{thm: SA.convergence} to obtain parameter-free optimal convergence rates for $Q$-learning. Our analysis here targets two specific settings in which such rates are still lacking: the synchronous variant with average reward and the asynchronous version with exponential discounting (see appendix for empirical validation in synthetic problem setups). 

For any set $U,$ denote by $\Delta(U)$ the collection or set of probability distributions on $U.$

\paragraph{Synchronous Average-reward $Q$-learning:} We first show how our setup from Section~\ref{s: setup.main} specializes for this algorithm. Each agent $i$ controls a Markov Decision Process (MDP) $\cM_i := (\cS, \cA, \cP_i, \cR_i),$ where the state space $\cS$ and action space $\cA$ are common to all the agents, while $\cP_i: \cS \times \cA \to \Delta(\cS)$ and $\cR_i: \cS \times \cA \to \bR$ are agent-specific transition and reward functions, respectively. For a policy $\pi: \cS\to\Delta(\cA),$ the average reward of agent $i$ is 
\begin{equation}
    r^\pi_i := \liminf_{T\to\infty} \frac{1}{T} \bE\left[ \sum_{t=0}^{T-1} \cR_i(s_t,a_t)\right],
\end{equation}
where $a_t$ and $s_{t + 1}$ are sampled from $\pi(\cdot|s_t)$ and $\cP_i(\cdot|s_t,a_t),$ respectively. Agent $i$'s \textit{optimal policy} is
\begin{equation}
    \piS_i := \argmax_{\pi} (r^\pi_i).
\end{equation}

In average reward, one computes the optimal \textit{differential $Q$-value function} $\QS_i$ to find the optimal policy; $\QS_i$ is defined as follows. Let $\cT^J_i:\bR^{\SA}\to \bR^{\SA}$ satisfying
\begin{equation}\label{e: J.step.bellman}
    \cT^J_iQ := \cR_i + \sum_{k=1}^{J-1}(\cP_i^{\pQ})^k \cR_i + (\cP^{\pQ}_i)^JQ
\end{equation}
denote agent $i$'s \textit{differential $J$-step Bellman operator}, where $\pQ(s) = \argmax_{a} Q(s,a)$ and, for $s,s'\in\cS$ and $a,a'\in\cA,$
\begin{equation}
    \cP^{\pQ}_i(s',a'|s,a) = \begin{cases} \cP_i(s'|s,a) & \text{ if } a' \in \pQ(s'), \\ 0 & \text{ otherwise.} \end{cases}
\end{equation} 
%
% $\QS_i$ 
Choosing  $J\geq |\cS|$ guarantees that $\cT^J_i$ is a contraction in the span semi-norm $\|\cdot\|_{\spf}$ (see the discussion below \cite[Assumption~3]{zhang2021finite}). Now, for $\p{} = \|\cdot\|_{\spf},$ the subspace $E$ from \eqref{e:norm.vanishing.subspace} equals $\{c\ones : c \in \bR\},$ where $\ones$ is the all ones vector in $\bR^{|\cS| |\cA|}.$ 
% Since $\|\cdot\|_{\spf}$ is a semi-norm contraction, 
% EThe semi-norm contraction implies 
%
% We can show that under assumption \ref{a: reachability}, that for each $i,$ the operator $T_i$ is a semi-norm contraction w.r.t the span semi-norm $\p{\cdot\p{_{\spf}$ . Further, 
Therefore, there exists a vector $Q \in \bR^{\SA}$ satisfying
\begin{equation}
    \cT^J_i Q - Q \in \{c\ones: c\in \bR\}.
\end{equation}
We denote such a vector by $\QS_i.$ The optimal policy $\piS_i$ is greedy with respect to $\QS_i,$ i.e., $\piS_i = \pi_{\QS_i}.$ 

As shown above \eqref{e:H.semi-norm.contraction.factor}, the averaged map $\cT^J:=\frac{1}{N}\sum_{i=1}^{N}\cT^J_i$ is also a semi-norm contraction under $\|\cdot\|_\spf.$ Hence, there exists a $\QS\in \bR^{\SA}$ such that
\begin{equation}
\label{e: AvgQ.fixed.point}
    \cT^J\QS - \QS \in \{ c\ones : c\in \bR\}.
\end{equation}

To compute such a $\QS,$ distributed synchronous $J$-step average-reward $Q$-learning can be used, and it is a natural extension of the single-agent variant in \cite[Algorithm 2]{zhang2021finite}. Its update rule is a special case of \eqref{e: SA.main.update}. We now specify the local operator $h_i$ and the local Markov chain $(\yti)$---because of synchronous updates, the latter reduces to IID samples. 

Set $d=\SA$ and let $\cY = \cS^{(\cS \times \cA) \times J}$. Thus each 
$y \in \cY$ specifies $y(s,a,k) \in \cS$ for every $(s,a) \in \cS \times \cA$ 
and $k \in \{1,\ldots,J\}$. At iteration $t$, the samples $y_t^i(s,a,k)$ are 
drawn according to $y_t^i(s,a,k) \sim \cP_i(\cdot \mid s,a)$, independently 
across $k$, $(s,a)$, and $t$. Consequently, $(y_t^i)$ is IID across iterations 
and independent of the iterate sequence $(Q_t)$.

Given $Q \in \bR^{\SA}$ and $y \in \cY$, the $J$-step rollout associated with 
a state-action pair $(s,a)$ is defined recursively as follows. Let 
$s_0 = s$ and $a_0 = a$. For $k = 1,\ldots,J$, set $s_k = y(s_{k-1}, \pQ(s_{k-1}), k).$ 

The local driving function is then defined by
\begin{multline}\label{e: avgQ.substitute}
    h_i(Q, y)(s,a)   := \cR_i(s,a) 
    \\
    + \sum_{k=1}^{J-1}\cR_i(s_k, \pQ(s_k))
    + \max_{a'} Q(s_J, a').
\end{multline}

Our next main result obtains parameter-free optimal convergence rates for the Polyak-Ruppert average $(\bQ_t)$ of the synchronous average-reward $Q$-learning.
% With the above definitions of $h_i$ and $(\yti),$ we have the following result about 
%
\begin{theorem}[Synchronous Average-Reward Q-learning]
\label{thm: avgQ.convergence}
    For each $i \in [N],$ the local function $h_i$ and the Markov chain $\{(\yti): i\in [N]\}$ satisfy the conditions \ref{a: SA.update.function} and \ref{a: SA.noise}. Moreover, the raw iterates satisfy \ref{a: raw.iterate.convergence.rate}. Hence, the conclusion of Theorem \ref{thm: SA.convergence} holds, i.e., $\bE\|\bQ_T-\QS\|_{\spf}=\tilde{O}(1/\sqrt{NT}).$  
\end{theorem}

% \begin{theorem}[Average-reward $Q$-learning]\label{thm: convergence.avgfed}
%     %
%     Consider the distributed $Q$-learning algorithm from \eqref{e: avgQ} and let $(\bQ_t)$ be its Polyak-Ruppert average given by \eqref{e : SA.PR.update}. Suppose condition \ref{a: reachability} holds and let $\at= \frac{1}{(t+1)^\alpha}$ for some $\alpha \in (1/2,1).$  Then, for all $T > 0,$
%     \[
%         \bE\|\bQ_T - \QS\|^2_{\spf} \leq \frac{\Cal}{N(T+1)} + \frac{\Caq}{(T+1)^{2\alpha}},
%     \]
%     %
%     where the constants $\Cal$ and $\Caq$ are as defined in Table \ref{tab: constants} in the Appendix.
% \end{theorem}

\begin{remark}
    \cite{zhang2021finite} derive optimal convergence rates for the synchronous average-reward $Q$-learning in the single-agent case. However, the stepsize $\at$ there depends on $\cT_i^J$'s contraction factor, which is unknown as it depends on $\cP_i,$ $i \in [N].$ In contrast, we obtain optimal rates with parameter-free stepsizes. 
\end{remark}

\begin{remark}\label{rem: hetero.gap}
    Let $\ep,\er>0$ be  such that, for every pair of agents $i,j$ and every state-action pair $(s,a),$
    \[
        \|\cP_i(\cdot|s,a) - \cP_j(\cdot|s,a)\|_\TV \leq \ep \|\cP_i(\cdot|s,a)\|_\TV
    \]
    and $\|\cR_i - \cR_j\|_\infty \leq \er.$ Then, we can show that for each agent $i,$ $\|\QS-\QS_i\|\leq \Cah(\ep,\er),$ implying that $(\bQ_T)$ converges to $\QS_i$ modulo the heterogeneity gap, i.e., $\Cah(\ep,\er).$ This gap goes to $0$ as $\ep + \er \to 0.$
\end{remark}

\paragraph{Asynchronous Exponentially-Discounted $Q$-Learning:} 
The setup here is similar to one in the average-reward case with an additional parameter $\gamma\in (0,1),$ called the discount factor. Given a policy $\pi:\cS\to\Delta(\cA),$ agent $i$'s $Q$-value function $Q^\pi_i\in \bR^{\SA}$ is now defined as
\begin{equation}
    Q^\pi_i(s,a) := \bE_\pi\left[\sum_{t=0}^{\infty}\gamma^t\cR_i(s_t,a_t) \bigg| s_0=s, a_0=a\right],
\end{equation}
where, for $t \geq 1,$ $a_t\sim \pi(\cdot|s_t)$ and $s_{t+1}\sim\cP_i(\cdot|s_t,a_t).$
Also, agent $i$'s optimal policy $\piS_i$ is one that satisfies
%
% that maximizes the $Q$-value function, i.e., for any policy $\pi$ and state-action pair $(s,a),$ 
%
\[
    Q^{\piS_i}_i(s,a)\geq Q^\pi_i(s,a), \quad \forall \pi \text{ and } \forall s,a.
\]
Each agent $i$ computes $\piS_i$ by estimating $Q^{\piS_i}_i,$ which is a fixed-point of the Bellman operator $\cT_i$ given by $  \cT_iQ = \cR_i + \gamma\cP^{\pQ}_iQ,$
where the greedy policy $\pQ$ and the transition matrix $\cP^{\pQ}_i$ have the same meaning as in the average-reward case. For each agent $i$, $\cT_i$ is a $\gamma$-contraction in $\|\cdot\|_\infty$ and has the unique fixed point $\QS_i\in \bR^{\SA};$ importantly, $\piS_i = \pi_{\QS_i}.$ As in \eqref{e:H.semi-norm.contraction.factor}, the averaged map $\cT_i := \frac{1}{N}\sum_i \cT_i$ also is a $\gamma$-contraction and has a unique fixed point $\QS\in \bR^{\SA}.$ 

One way to find this $\QS$ is to use distributed asynchronous exponentially-discounted $Q$-learning, which we describe next. Let $\mu$ be a fixed behavior policy. Then, the above algorithm is obtained from \eqref{e: SA.main.update} by taking $d=\SA,$ $\cY=\cS \times \cA \times\cS,$ and letting
\begin{align}\label{e: expQ.substitute}
    h_i(Q,s,a,s') = {} &  Q - Q(s,a)\ones_{(s,a)} 
    \nonumber\\
    & + [\cR_i(s,a) + \gamma \max_{a'} Q(s',a') ]\ones_{(s,a)}
    \nonumber
\end{align}
and $\yti  = (s^i_t, a^i_t, s^i_{t + 1}),$ where for each agent $i$ and each $t\geq 0,$ $a^i_t\sim\mu(\cdot|s^i_t)$ and $s_{t+1}^i\sim \cP_i(\cdot| s^i_t, a^i_t).$ 

Our next main result establishes parameter-free rates for the Polyak-Ruppert average $(\bQ_t)$ of the exponentially discounted $Q$-learning algorithm described above. 
\begin{theorem}[Asynchronous Exponentially-Discounted Q-learning]\label{thm: expQ.convergence}
    The assumptions \ref{a: SA.update.function}--\ref{a: raw.iterate.convergence.rate} hold. Consequently, $\bE\|\bQ_T - \QS\|_\infty = \tilde{O}(1/\sqrt{NT}).$
\end{theorem}

\section{Proof Outline}
\label{s: proof.main}
In this section, we sketch our proofs for Theorems~\ref{thm: SA.convergence}, \ref{thm: avgQ.convergence}, and \ref{thm: expQ.convergence}. The detailed proofs can be found in the Appendix.

\subsection{Proof (Sketch) of Theorem \ref{thm: SA.convergence}}
\label{s: proof.SA}
%
% For pedagogical ease to identify the challenges that arise and modifications that should be made when the operator is contractive in a semi-norm, we briefly outline \citet{li2023statistical}’s method to analyze the Polyak–Ruppert average $(\bQ_t)$ of a nonlinear fixed-point iteration $(Q_t)$ under a contractive {\em mean} operator. The key idea is to construct two auxiliary sequences, $(L_t)$ and $(U_t),$ {\color{red} VG: maybe sequences should be with curly brackets?} whose Polyak--Ruppert averages sandwich the error in $(\bQ_t)$ in a coordinate-wise sense. Both $(L_t)$ and $(U_t)$ correspond to a linear stochastic approximation plus a rapidly vanishing remainder; thus, classical Polyak--Ruppert theory shows that they enjoy $\tilde{O}(1/(NT))$ mean-squared bounds. The norm's monotonicity then allows these bounds to transfer directly to $\bQ_t.$ 

% For semi-norms, one could imagine constructing a similar argument. However, this argument would fail since coordinate-wise order does not imply semi-norm order: $0 \leq x \leq y \nRightarrow \|x\| \leq \|y\|.$ Hence, the sandwiching idea cannot be used here.

% Our proof, instead, pairs the semi-norm with its monotone norm counterpart, which is guaranteed to exist thanks to Lemma \ref{lem:semi-norm.norm.coupling} and Assumption \ref{a: semi-norm.norm.relation}. This allows us to exploit the monotonicity of the norm alongside the contraction property enjoyed by the semi-norm. Specifically,
Our proof proceeds through four key steps:
\begin{enumerate}
\item Derive a recursion for the raw nonlinear iterate error, involving a linear core and a nonlinear perturbation.

\item Extend this decomposition to the Polyak–Ruppert average, expressing the nonlinear perturbation as a sum of matrix–vector products.

\item \textbf{Leverage semi-norm contraction} and geometric mixing (\ref{a: SA.noise}) to bound the average of the linear components.

\item Bound the nonlinear term involving the sum of matrix-vector products via three sub-steps:
    \begin{enumerate}[leftmargin=*, align=left]
    \item Bound its semi-norm by a sum of matrix-vector semi-norm products, and then \textbf{bound the semi-norm of each vector by its induced norm}.   
    
    \item \textbf{Use semi-norm contraction} to obtain a uniform bound on the matrix semi-norm.
    
    \item \textbf{Invoke the induced norm's equivalence with a monotone norm} to prove that the overall nonlinear term vanishes rapidly, yielding the desired rate.
    \end{enumerate}
\end{enumerate}
%
% Our Step 1 follows \citet{li2023statistical}, who study nonlinear stochastic fixed-point iterations when the mean operator is a contraction under the $\|\cdot\|_\infty$-norm. From Step 2 onward, however, our method diverges sharply to address the semi-norm’s non-monotonicity. To elaborate, Li et al. construct two auxiliary sequences, $(L_t)$ and $(U_t),$ whose Polyak-Ruppert averages sandwich $\bQ_t$'s error coordinate-wise. Since these auxiliary sequences equal a linear stochastic approximation plus a rapidly vanishing remainder, classical Polyak–Ruppert theory shows that they enjoy a $\tilde{O}(1/(NT))$ mean-squared bounds. The monotonicity of $\|\cdot\|_\infty$-norm---i.e., $0 \leq x \leq y \implies \|x\| \leq \|y\|$---then lets these bounds pass directly to $\bQ_t.$ However, this translation fails under a semi-norm---coorindate-wise order does not translate into semi-norm order.

We now describe the above four steps in detail.

\paragraph{Step 1 (Raw Error Decomposition).} Using the definitions of $b_i$ and $A_i$ from Assumption~\ref{a: SA.update.function}, we introduce a few notations. For $t \geq 0$ and $Q \in \bR^d,$ define the $d$-dimensional vectors $\bht^Q = \frac{1}{N}\sum_{i=1}^N b_i(Q,\yti)$ and $b^Q  = \frac{1}{N}\sum_{i=1}^N \bE_{y\sim\eta_i}[b_i(Q,y)],$ and the $d\times d$ matrices
\[
    \Aht^Q = \frac{1}{N}\sum_{i=1}^N A_i(Q,\yti), \quad A^Q = \frac{1}{N}\sum_{i=1}^N\bE_{y\sim\eta_i}[A_i(Q,y)].
\]
Using \eqref{e:H_i.defn}, \eqref{e:average.map.defn}, and the structural form of $h_i$ in Assumption~\ref{a: SA.update.function}, it follows that $H$ has the form $H(Q) = b^Q + A^{Q}Q.$ Since $\QS$ is a fixed point of $H,$ we then have from \eqref{e:semi.norm.fixed.point} that 
\begin{equation}\label{e: fixed.point}
    b^{\QS} + A^{\QS}\QS = H(\QS) =  \QS + \eS
\end{equation}
for some $\eS\in E.$ Finally, for $t \geq 0,$ define 
\[
    \Delta_t := Q_t - (\QS + \eS) \quad \text{and} \quad \bDelta_t := \bQ_t - (\QS+\eS).
\]
Since $\eS \in E,$ we have $\p{\bDelta_T} = \p{\bQ_T - \QS}.$

The following result decomposes the raw iterate error $\Delta_t$ into the desired linear and non-linear components. 
% Our goal is to prove that $\p{\bDelta_T}$ decays to zero at a rate $\tilde{O}(1/T).$ We begin with an update rule for $(\Delta_t)$ derived from \eqref{e: SA.main.update}.
%
\begin{lemma}
\label{lem:raw.iterate.error.split}
    For any $t \geq 0,$ we have
    \begin{equation}\label{e: SA.update.modified}
        \Delta_{t+1} = \left[\bI - \at \left(\bI- \Aht^{\QS}\right)\right] \Delta_t + \at\omega_t + \at \xi_t,
    \end{equation}
    where $\omega_t := \left[\bht^{\QS} - b^{\QS}\right] + \left[\Aht^{\QS} - A^{\QS}\right]\QS + \eS,$ and 
    $\xi_t := \left[\bht^{Q_t} - \bht^{\QS}\right] + \left[\Aht^{Q_t} - \Aht^{\QS}\right]Q_t.$
\end{lemma}
\begin{remark}
\label{rem:raw.iterate.error.split}
    The $\at \xi_t$ term in \eqref{e: SA.update.modified} stems from the nonlinearity of the fixed-point iteration \eqref{e: SA.main.update}. If this term is omitted, \eqref{e: SA.update.modified} reduces to a standard linear stochastic approximation.
\end{remark}

\paragraph{Step 2 (Polyak-Ruppert Error Decomposition).} 
We now use the error decomposition achieved in \eqref{e: SA.update.modified} to derive a similar decomposition for $\bDelta_t$---the Polyak-Ruppert average of $\Delta_t$---in Lemma~\ref{lem: PR.avg.error.split} below. This refined decomposition underpins the subsequent analysis, allowing us to exploit the contraction property of the semi-norm alongside the equivalence between the induced norm and any monotone norm.

For $0 \leq t_1 \leq t_2,$ let $
    \Gamma_{t_1:t_2} := \prod_{t=t_1}^{t_2-1}\left[ \bI - \at [\bI - \Aht^{\QS}] \right].$

\begin{lemma}
\label{lem: PR.avg.error.split}
    Let $\tstr$ be defined as in Theorem~\ref{thm: SA.convergence}. Then, for $T \geq \tstr,$
    \begin{equation}\label{e: PR.decomp}
        \bDelta_T = \bDeltrans_T 
        + \bDelinit_T + \bDelnoise_T + \bDelnlin_T,
    \end{equation}
    where 
    
    $\bDeltrans_T := \bigg(  \frac{1}{T}\sum\limits_{t=0}^{\tstr-1}  \Delta_t \bigg), \quad \bDelinit_T  := \frac{1}{T}\sum\limits_{t=\tstr}^{T-1} \Gamma_{0:t} \Delta_0,$
    
    \noindent and
    \begin{align*}
        \bDelnoise_T & := \frac{1}{T}\sum\limits_{t=\tstr}^{T-1} \sum_{k=0}^{t-1}\ak\Gamma_{k+1:t} \omega_k, 
        \\
        \bDelnlin_T & := \frac{1}{T}\sum\limits_{t=\tstr}^{T-1} \sum_{k=0}^{t-1}\ak\Gamma_{k+1:t}\xi_k. 
    \end{align*}
    %
    % \hspace{0.5em}  $\bDeltrans_t := \bigg(  \frac{1}{T}\sum\limits_{t=0}^{\tstr-1}  \Delta_t \bigg),$ \hspace{1em} $\bDelinit_t := \frac{1}{T}\sum\limits_{t=\tstr}^{T-1} \Gamma_{0,t} \Delta_0,$
    % %
    % \begin{align*} 
    %     % \bDeltrans_t & := \bigg(  \frac{1}{T}\sum\limits_{t=0}^{\tstr-1}  \Delta_t \bigg), \\
    %     %
    %     % \bDelinit_t  & := \frac{1}{T}\sum\limits_{t=\tstr}^{T-1} \Gamma_{0,t} \Delta_0, \\
    %     %
    %     \bDelnoise_t & := \frac{1}{T}\sum\limits_{t=\tstr}^{T-1} \sum_{k=0}^{t-1}\ak\Gamma_{k+1:t} \omega_k, \\
    %     %
    %     \intertext{ and }
    %     %
    %     \bDelnlin_t & := \frac{1}{T}\sum\limits_{t=\tstr}^{T-1} \sum_{k=0}^{t-1}\ak\Gamma_{k+1:t}\xi_k. 
    % \end{align*}
    % %
    In turn, we have that
    \begin{equation}
    \label{e:PR.error.decomposition}
        \bE\p{\bDelta_T} \leq \bAtrans_T + \bAinit_T + \bAnoise_T + \bAnlin_T,
    \end{equation}
    where 
    \[
        \begin{aligned}
            \bAtrans_T & := \frac{1}{T}\sum_{t=0}^{\tstr-1}\bE\p{\Delta_t}, & \bAinit_T & := \bE\p{\bDelinit_T},\\
            \bAnoise_T & := \bE\p{\bDelnoise_T}, & \bAnlin_T & := \bE\p{\bDelnlin_T}.
        \end{aligned}
    \]
\end{lemma}

\begin{remark}
    The superscripts \textnormal{trans}, \textnormal{init}, \textnormal{noise}, and \textnormal{nonlin}  denote, respectively, the transient, initial-condition-based, noise-induced, and nonlinear components. 
\end{remark}

\begin{remark}
    In the spirit of Remark~\ref{rem:raw.iterate.error.split}, the $\bAnlin_T$ term in \eqref{e:PR.error.decomposition} stems from the nonlinearity in \eqref{e: SA.main.update}. If this ter, were omitted, \eqref{e:PR.error.decomposition} would reduce to a standard decomposition in PR-averaged linear stochastic-approximation literature analyses, e.g., see \cite[equation~(16)]{durmus2025finite}.
\end{remark}

\paragraph{Step 3 (Bounding the Averaged Linear Components).} 
Now, we use the semi-norm contraction to bound the terms $\bAtrans_T,$ $\bAinit_T,$ and $\bAnoise_T$, that are present in \eqref{e:PR.error.decomposition}. 

% To prove Theorem \ref{thm: SA.convergence}, we need to bound the terms $\bAtrans_T,$ $\bAinit_T,$ $\bAnoise_T,$ and $\bAnlin_T$ appropriately. 

% The subsequent lemmas will now bound these terms one by one.
% From \eqref{e: SA.update.modified}, a telescopic expansion gives
% %
% \[
%     \Delta_T = \Delta_0 + \sum_{t = 0}^{T - 1} \at \Big[ - (\bI - \Aht^{\QS}) \Delta_t + \omega_t + \xi_t \Big].
% \]
% The semi-norm triangle inequality then shows that
% %
% \[
%     \p{\Delta_T} \leq \p{\Delta_0} + \sum_{t = 0}^{T-1} \alpha_t \big[1 + \p{\Aht^{\QS}}] \p{\Delta_t} + \p{\omega_t} + \p{\xi_t} \big]
% \]
% %
% From \ref{a: SA.update.function}.(c), we have $\p{\Aht^{\QS}} \leq \CA$
%

First, by applying the discrete Gronwall inequality on the $\Delta_t$ expansion in \eqref{e: SA.update.modified} and then using the fact that $\tstr$ is a constant, we get the following bound on $\bAtrans_T.$  

\begin{lemma}[\textbf{Bounding $\bAtrans_T$}]
\label{thm: PR.trans.bound}
    For $T>\tstr,$ $\bAtrans_T \leq \frac{\CD\tstr}{T},$ where $\CD>0$ is as defined in Table~\ref{tab: constants}.
    %
    % \[
    %     \bAtrans_T \leq \frac{\CD\tstr}{T},
    % \]
    %
\end{lemma}

Next, we get bounds on $\bAinit_T$ and $\bAnoise_T.$ For $t\geq0,$ let $\bE[\cdot|\cF_t]$ be denoted by $\bE_t[\cdot].$ We proceed with a bound on
$\bE_{t_1} \p{\Gamma_{t_1: t_2}},$ which is given in Lemma \ref{thm: SA.bound.matrix} below. To prove this result, we follow the recipe in \cite[Proposition 7]{durmus2025finite}, which proves a similar result for \textit{linear stochastic iterations} with a fixed stepsize $\alpha.$ A key difference is that their result relies on a \textit{norm bound} $\|\bI-\alpha[\bI - A]\|< 1-\alpha(1-\lambda)$ for some  $\lambda>0,$ whereas our Lemma \ref{thm: SA.bound.matrix} uses an analogous semi-norm bound. 

The proof proceeds by dividing $\Gamma_{t_1:t_2}$ into blocks of a fixed size $\kappa > 0.$ Specifically, we define the $\ell$-th block as
\[
    Y_\ell:= \prod_{t=t(\ell-1)}^{t(\ell) -1}\Big[\bI - \at[\bI-\Aht^{\QS}]\Big], \text{ for } \ell \leq \floor{\frac{t_2-t_1}{\kappa}},
\]
with $t(0) = t_1$ and $t(\ell)= t(\ell-1) + \kappa -1.$ To bound $\p{\Gamma_{t_1:t_2}},$ we then bound each block. To get the latter, block $Y_\ell$ is decomposed into a contractive part---bounded using the semi-norm contraction, and a remainder part---bounded using the geometric mixing (Assumption \ref{a: SA.noise}). 

In particular, we write
$Y_\ell = \big[\bI - \alpha_{\ell-1:\ell}[\bI-A^{\QS}]\big] + R_\ell,$
where $\alpha_{\ell-1:\ell}:=\sum_{s=t(\ell-1)}^{t(\ell)-1}\alpha_s,$ and $R_\ell$ is the remainder term comprising two types of terms: fluctuation terms involving
$(\Aht^{\QS}-A^{\QS})$ and higher-order products of the stepsizes. Using geometric mixing (\ref{a: SA.noise}) and the bounds in \ref{a: SA.update.function}(c), it is easy to see that
\[
    \bE_{t_1}\p{R_\ell} \leq \alpha_{t(\ell)}\Big(\frac{C_E\CA}{1-\rho}\Big) + \alpha^2_{\ell(t-1)}(1+\CA)^\kappa.
\]
On the other hand, for the first term in the $Y(\ell)$ expansion, we utilize the semi-norm contractivity of $H$ and \ref{a: SA.update.function}.(d) to conclude  $\p{A^{\QS}} \leq \beta.$ Thereafter, we use $\alpha_{\ell-1:\ell}\geq \kappa \alpha_{t(\ell)}$ for $t > \tstr$ to get
\[
%\label{e: matrix.block.bound.quad}
    %
    \p{\bI - \alpha_{\ell-1:\ell}[\bI-A^{\QS}]} 
    \leq 1 - \kappa\alpha_{t(\ell)}(1-\beta).
    % \alpha_{t(\ell)}.
    %
\]

By combining the above two expressions, we get
\begin{equation}\label{e: bound.matrix.block}
    \bE_{t_1}\p{Y_\ell} \leq 1 - \beta_\kappa\alpha_{t(\ell)} + \alpha^2_{t(\ell-1)}(1+\CA)^\kappa,
\end{equation}
where $\kappa$ is chosen to ensure that $\beta_\kappa := \kappa (1-\beta)- \frac{\CE\CA}{1-\rho}>0.$ The following bound follows from \eqref{e: bound.matrix.block}. 

\begin{lemma}\label{thm: SA.bound.matrix}
    Let $\Cgm>0$ be as in Table~\ref{tab: constants}. For $0\leq t_1<t_2,$
    \[
        \bE_{t_1}\p{\Gamma_{t_1:t_2}} \leq \Cgm \ e^{-\beta_\kappa\sum_{\ell=0}^{\floor{(t_2-t_1)/\kappa}}\alpha_{t_1+\ell \kappa}}.
    \]
\end{lemma}

\noindent The following bound on $\bAinit_T$ now follows from Lemma \ref{thm: SA.bound.matrix}.

\begin{lemma}[\textbf{Bounding} $\bAinit_T$]
\label{thm: PR.init.bound}
    For $T>\tstr,$ $\bAinit_T \leq \frac{\xi_\Gamma \p{\Delta_0}}{T}.$
\end{lemma}
% {\color{red} VG: should we say something like the preceding sentence for Lemma 4.8?}

Next, we obtain a bound on $\bAnoise_T$ using a recipe similar to the one used to bound $\Delta_t^{(2)}$ in \cite{naskar2025}, which itself builds upon the proofs for \cite[Section 2, Propositions 8 and 10]{durmus2025finite}. In particular, from Condition~\ref{a: SA.update.function}.(a), we have 
\[
    \bDelnoise_T = \hDelnoise_T  + \frac{1}{T} \sum_{t = \tstr}^{T - 1} \sum_{k = 0}^{t - 1} \alpha_k \eS,
\]
where $\hDelnoise_T = \frac{1}{T}\sum_{t=\tstr}^{T-1} \hDel^{(2)}_t,$ $\hDel^{(2)}_t = \sum_{k=0}^{t-1}\ak\Gamma_{k+1:t} \omega'_k$ and $\omega'_k  = \omega_k - \eS.$ Hence, $\p{\bDelnoise_T} = \p{\hDelnoise_T}.$ Now, $\hDel^{(2)}_t$ is of form given in ~\cite[(23)]{naskar2025federated}. Hence, following Lemma~IV.7's proof there, the below result follows.

% Clearly,
% %
% \[
%     \Delnoise_t := \sum_{k=0}^{t-1}\ak  \Gamma_{k + 1:t}\ \omega_k,
% \]
% %
% we have
% %
% \begin{align}
%     %
%     \Delnoise_{t + 1} = {} & [\bI - \at (\bI - \Aht^{\QS}))] \Delnoise_t +  \alpha_t \omega_t \\
%     %
%     = {} & [\bI - \at (\bI - A^{\QS})] \Delnoise_t +  \alpha_t \omega_t \\
%     %
%     {} & + \alpha_t [\Aht^{\QS} - A^{\QS}] \Delnoise_t
%     %
% \end{align}
% satisfies
% %
% \[
%     \Delnoise_t = 
% \]

% we write
% %
% \begin{multline*}
%     \bDelnoise_T = \frac{1}{T}\sum\limits_{t=\tstr}^{T-1} \sum_{k=0}^{t-1}\ak \prod_{\ell = k + 1}^{t - 1} \big[\bI - \alpha_\ell [\bI - A^{\QS}] \big]  \omega_k \\
%     %
%     + \frac{1}{T}\sum\limits_{t=\tstr}^{T-1} \sum_{k=0}^{t-1}\ak  \left[\Gamma_{k + 1: t}  - \prod_{\ell = k + 1}^{t - 1} \big[\bI - \alpha_\ell [\bI - A^{\QS}]\big] \right]  \omega_k
% \end{multline*}
% %
% We then bound the first term using the contractive nature of $[\bI-\at[\bI-A^{\QS}]]$ and the geometric mixing of $\omega_t.$ 

% \[
%     \prod_{\ell = k + 1}^{t - 1} [1 - (1 - \beta) \alpha_\ell]
% \]

% ---while the second term involves a sum of products between random matrices $\Aht^{\QS},$ contractive matrices $[\bI - \at[\bI-A^{\QS}]],$ and matrix noise of the form $[\Aht^{\QS }-A^{\QS}]$---to be split into appropriately spaced intervals to invoke independence between $\Aht^{\QS}$ and the matrix noise, and then bounded using geometric mixing.

% The following lemma provides a bound on $\bAnoise_T.$ 
%
\begin{lemma}[\textbf{Bounding} $\bAnoise_T$]
\label{thm: PR.noise.bound}
    For $T>\tstr,$
    \[
         \bAnoise_T \leq \frac{C^\text{noise}_1\sqrt{\tauT}}{\sqrt{NT}} + \frac{C^\text{noise}_2\ln(T)}{T^\alpha}.
    \]
    where $C^\text{noise}_1, C^\text{noise}_2$ are as in Table~\ref{tab: constants}.
\end{lemma}

\paragraph{Step 4 (Analysis of the Nonlinear term).} 
Finally, we obtain a bound on the nonlinear perturbation term $\bAnlin_T$ in 
Lemma \ref{thm: PR.nlin.bound} below. To get that bound, we first switch the double sum in the $\bDelnlin_T$ expression from Lemma \ref{lem: PR.avg.error.split} to get 
\[
    \bDelnlin_T = \frac{1}{T}\sum_{k=0}^{T-2}\bigg(\ak\sum\limits_{t=k+1}^{T-1} \Gamma_{k+1:t}\bigg)\xi_k 
    =  \frac{1}{T}\sum_{k=0}^{T-2}M_k^T \xi_k,
\]
where, for  $0\leq t_1 < t_2,$ $M_{t_1}^{t_2}  := \alpha_{t_1}\sum\limits_{s = t_1+1}^{t_2-1}\Gamma_{t_1+1:s},$ 
Now, by applying the semi-norm triangle  inequality, we get
\begin{align}\label{e: semi-norm.matrix-vector.prod}
    \bAnlin_T \leq \frac{1}{T}\sum_{k=0}^{T-2} \bE\left[ \p{M^T_{k}}\ \p{\xi_k}\right].
\end{align}
Now, to bound \eqref{e: semi-norm.matrix-vector.prod}, we show that $\bE_k\p{M^T_k}=O(1)$ and that $\bE\p{\xi_k} = O(\tau_t \alpha_t);$ the latter term thus decays rapidly.

The following result applies Lemma~\ref{thm: SA.bound.matrix} to obtain the required bound on $\bE_k \p{M_k^T}.$

\begin{lemma}\label{e: matrix.uniform.bound.}
    For $0\leq t_1<t_2,$  we have $\bE_{t_1}\p{M^{t_2}_{t_1}} \leq \KG,$   where the constant $\KG>0$ is as defined in Table~\ref{tab: constants}.
\end{lemma}

Next, we obtain a bound on $\p{\xi_k}.$ Given that Assumption \ref{a: raw.iterate.convergence.rate} states $\bE\p{Q_k - \QS}^2 = \tilde{O}(\ak),$ it is enough for us to show that $\p{\xi_k} = O(\p{Q_k - \QS}^2).$ 

As a first step, we utilize the following sandwiching relation obtained from \ref{a: SA.update.function}.(b):
\begin{multline}\label{e: nonlin.sandwich}
    0 \leq \xi_k - [\bhk^{Q_k} - \bhk^{\QS}] \leq [\Ahk^{Q_k} - \Ahk^{\QS}][Q_k - \QS].
\end{multline}
\citep{li2023statistical} exploits a similar sandwiching relation (without the $[\bhk^{Q_k} - \bhk^{\QS}]$ term) to bound a certain remainder term. However, their work concerns norm contractive operators, where the monotonicity of their norm permits them to translate \eqref{e: nonlin.sandwich} to a desired norm-based inequality. This tactic fails for us since our semi-norm $\p{\cdot}$ lacks this monotonicity property. To overcome this hurdle, we turn to the induced norm $\|\cdot\|$ given in Lemma~\ref{lem:semi-norm.norm.coupling}.

Let $\|\cdot\|_m$ be a monotone norm satisfying \eqref{e: span.semi.norm.monotone.induced.norm}. Then,  
\[
\|\xi_k\|_m \leq \|\bhk^{Q_k} - \bhk^{\QS}\|_m 
    +  \| [\Ahk^{Q_k} - \Ahk^{\QS}][Q_k - \QS]\|_m.
\]
Combing this with $\p{\xi_k}\leq \|\xi_k\|$ and \eqref{e: span.semi.norm.monotone.induced.norm}  gives
\begin{multline}\label{e: nlin.quadratic.norm.semi-norm}
    \p{\xi_k} \leq \frac{c_u}{c_\ell}\|\bhk^{Q_k} - \bhk^{\QS}\| \\
    +  \frac{c_u}{c_\ell}\| [\Ahk^{Q_k} - \Ahk^{\QS}][Q_k - \QS]\|.
\end{multline}
% Before we proceed to derive \eqref{e: nlin.quadratic.norm.semi-norm}, we bound $M^T_k$ in the following lemma, which is a direct consequence of Lemma \ref{thm: SA.bound.matrix}. 
%

% %
% From the above lemma, 
% %
% \begin{align}\label{e: semi-norm.norm.bound}
%     %
%     \bE \big[ \p{M^T_{k}}& \cdot \|\xi_k\|\big]
%     \nonumber\\
%     & \overset{(a)}{=} \bE\bE_k \left[ \p{M^T_{k}}\cdot \|\xi_k\|\right]
%     \nonumber\\
%     & \overset{(b)}{=}\bE\left[ \bE_k\p{M^T_{k}}\cdot\|\xi_k\|\right] \leq \KG\bE\|\xi_k\|,
%     %
% \end{align}
% %
% where $(a)$ uses the tower property of expectation and $(b)$ uses the fact that $\|\xi_k\|$ is $\cF_k$-measurable.
%
%
To conclude \textbf{Step 4}, we show that the RHS in \eqref{e: nlin.quadratic.norm.semi-norm} is  $O(\p{Q_k-\QS}^2).$ For the first term, we use the following fact from Assumptions \ref{a: SA.update.function}.(b) and \ref{a: SA.update.function}.(d):
\[
    \|\bhk^{Q_k} - \bhk^{\QS}\| \begin{cases} = 0 & \text{if } \p{Q_k - \QS}<\cstr, \\ \leq 2\Cb &  \text{otherwise.}  \end{cases}
\]
This directly implies that
\begin{equation}
\label{e:bhk.bound}
\|\bhk^{Q_k} - \bhk^{\QS}\| \leq {} (2\Cb/\cstr^2)\ \p{Q_k - \QS}^2.
\end{equation}
While a similar argument shows $\|\Ahk^{Q_k} - \Ahk^{\QS}\| = O(\p{Q_k - \QS}^2),$ applying the naive inequality $\|B x\|\leq \|B\|\|x\|$ for the second term in \eqref{e: nlin.quadratic.norm.semi-norm} yields an expression involving $\|x\|$---in our case, $\|Q_k - \QS\|$---a potentially unbounded term! 

Instead, we need a bound involving $\p{Q_k - \QS}.$ To that end, let $e_k := \argmin_{e\in E}\|Q_k -\QS-e\|.$
Since $e_k \in E,$ Assumption \ref{a: SA.update.function}.(a) implies $[\Ahk^{Q_k}-\Ahk^{\QS}]e_k=0.$ Hence,
\[
    [\Ahk^{Q_k} - \Ahk^{\QS}]\ [Q_k - \QS]
    = [\Ahk^{Q_k} - \Ahk^{\QS}]\ [Q_k - \QS - e_k].
\]
Now, from Lemma \ref{lem:semi-norm.norm.coupling}, we know $\|Q_k -\QS-e_k\| = \p{Q_k - \QS}.$ 
Hence, by using arguments similar to \eqref{e:bhk.bound}, we get
% we also have
\begin{equation}\label{e: nlin.seminorm.mat.bound}
    \|\Ahk^{Q_k} - \Ahk^{\QS}\|  
    \leq {} (2\CA/\cstr)\ \p{Q_k - \QS}.
\end{equation}
By combining \eqref{e: nlin.quadratic.norm.semi-norm}, \eqref{e:bhk.bound}, and \eqref{e: nlin.seminorm.mat.bound}, we get the desired bound on $\p{\xi_t}$:
\begin{align}
    &\p{\xi_k}
    \nonumber\\
    &\leq \frac{c_u}{c_\ell}\Big(\|\bhk^{Q_k} - \bhk^{\QS}\| + \|\Ahk^{Q_t} - \Ahk^{\QS}\|\ \|Q_k - \QS -e_k\|\Big)
    \nonumber \\
    & \leq \frac{2c_u}{c_\ell}\Big(\frac{\CA}{\cstr} + \frac{\Cb}{\cstr^2}\Big)\p{Q_k - \QS}^2.  \label{e:xi_k.Bd}
\end{align}
Together with Assumption~\ref{a: raw.iterate.convergence.rate}, the above expression implies that $\bE\p{\xi_k}=\tilde{O}(\ak).$
%
% \begin{equation}\label{e: nlin.seminorm.vec.final}
%     %
%     \bE\p{\xi_k} \leq 2\Big(\frac{\Cb}{\cstr^2} + \frac{\CA}{\cstr}\Big)\tauk\ak, \quad \forall k>\tstr.
%     %
% \end{equation}
%
The following lemma is a consequence of this fact combined with Lemma \ref{e: matrix.uniform.bound.} and \eqref{e: semi-norm.matrix-vector.prod}.
\begin{lemma}[\textbf{Bounding} $\bAnlin_T$]
\label{thm: PR.nlin.bound}
    For $T>\tstr,$
    \[
        \bAnlin_T \leq \frac{2c_u\KG C_Q}{c_\ell(1-\alpha)}\Big( \frac{\Cb}{\cstr^2}+\frac{\CA}{\cstr}\Big)\left(\frac{\tauT}{T^\alpha}\right).
    \]
\end{lemma}

Finally, putting together the bounds on $\bAtrans_T,$ $\bAinit_T,$ $\bAnoise_T,$ and $\bAnlin_T$ from Lemma \ref{thm: PR.trans.bound}, \ref{thm: PR.init.bound}, \ref{thm: PR.noise.bound}, and \ref{thm: PR.nlin.bound}, respectively, gives the desired bound in Theorem \ref{thm: SA.convergence}. \hfil \qed
%

%%%%%%%%%%%%%%%%%%%%%%%%%%%%%%%%%%%%%%%%%%%%%%%%%%%%%%%%%%%%%%%%%%%%%%%%%%%%%%%%%%%%%%%%%%%%%%%%%%%%%%%%%%%%

% First, let us represent this algorithm as a special case of \eqref{e: SA.main.update}. To do so, we set $d:=\SA,$ and let every $y\in \cY$ be a map from $\cS\times \cA\to\cS^J,$ that is, for every state-action pair $(s,a),$ $y(s,a)$ is a $J$-step trajectory of states $(s_1,\ldots, s_J).$ In other words, let $\cY:= \cS^{\cS\times\cA\times J}.$

% The Markov chain $(\yti)$ is defined as follows. For agent $i,$ let $y^i_t\in \cY$ such that for every state-action pair $(s,a),$ $\yti(s,a)$ is a  trajectory $(s^i_{t,1}, \ldots s^i_{t,J}),$ where the initial state $s^i_{t,1}$ is sampled from $\cP_i(\cdot| s,a)$ and for $1\leq k<J,$ the next state $s^i_{t,k+1}$ is sampled from $\cP_i(\cdot| s^i_{t,k}, \pi_{Q_t}(s^i_{t,k})).$ 
% %
% \begin{multline}
%     %
%     h_i(Q,y)(s,a) = \cR_i(s,a) 
%     \\
%     + \sum_{k=0}^{J-1}\cR_i(s_k, \pQ(s_k)) + \max_{a'} Q(s_J,a')
%     %
% \end{multline}

%%%%%%%%%%%%%%%%%%%%%%%%%%%%%%%%%%%%%%%%%%%%%%%%%%%%%%%%%%%%%%%%%%%%%%%%%%%%%%%%%%%%%%%%%%%%%%%%%%%%%%%%%%%%%%%%%%%%%%%%

\subsection{Proof (Sketch): Theorems~\ref{thm: avgQ.convergence} and \ref{thm: expQ.convergence}}
\label{s: proof.sketch.Application.RL}

In this section, we discuss the proof sketch for Theorem \ref{thm: avgQ.convergence}. Similar arguments are used to prove Theorem \ref{thm: expQ.convergence}. The details are provided in the Appendix. 

We only need to show that the $J$-step average-reward $Q$-learning, as a special case of \eqref{e: SA.main.update}, satisfies conditions \ref{a: SA.update.function}--\ref{a: raw.iterate.convergence.rate}. Due to IID sampling, \ref{a: SA.noise} is trivial, while \ref{a: raw.iterate.convergence.rate} follows from \cite[Eq.~B7]{zhang2021finite}. Also, one can show that there exist stochastic matrices $\{\hcP_{k}^{\pQ}\}_{k=1}^{J}$ such that $h_i(Q,y) = b_i(Q,y) + A_i(Q,y)Q,$ where $A_i=\hcP^{\pQ}_{J}$ and $b_i = \sum_{k=0}^{J-1}\hcP_{k}^{\pQ}\cR_i.$ Then, for this $A_i$ and $b_i$ definitions, conditions \ref{a: SA.update.function}.(a) and \ref{a: SA.update.function}.(c) follow from the stochasticity of $\{\hcP^{\pQ}_{k}\}.$ Condition \ref{a: SA.update.function}.(b) follows from the greedy property of $\pQ.$ Lastly, \ref{a: SA.update.function}.(d) follows from the fact that the greedy policy is piece-wise constant on $\bR^{\SA}.$

% To verify condition \ref{a: raw.iterate.convergence.rate}, we show the local function $h_i$ can be written as 

% in the form in \ref{a: SA.update.function} using suitable stochastic matrices $\{\hcP_{k}^{\pQ}\}_{k=1}^{J}.$ In particular, $A_i=\hcP^{\pQ}_{J}$ and $b_i$'s are given by sums of the form $\sum_{k=0}^{J}\hcP_{k}^{\pQ}\cR_i.$
% Conditions \ref{a: SA.update.function}.(a) and \ref{a: SA.update.function}.(c) then follow from the stochasticity of $\{\hcP^{\pQ}_{k}\}.$ Finally, \ref{a: SA.update.function}.(b) and \ref{a: SA.update.function}.(d) follow from the greedy property of $\pQ.$ The latter relies on the fact that the greedy policy is piece-wise constant on $\bR^{\SA}.$ 

%%%%%%%%%%%%%%%%%%%%%%%%%%%%%%%%%%%%%%%%%%%%%%%%%%%%%%%%%%%%%%%%%%%%%%%%%%%%%%%%%%%%%%%%%%%%%%%%%%%%%%%%%%%%

\section{Conclusions and Future Directions}
We prove parameter-free convergence guarantees for nonlinear stochastic fixed-point iterations whose mean operators are only semi-norm contractions. In particular, we show that the existing parameter-free (but suboptimal) bounds for raw iterates (see \ref{a: raw.iterate.convergence.rate}) can be improved to the optimal $\tilde{O}(1/T)$ rate by applying Polyak–Ruppert averaging. 

As an application, we obtain the first parameter-free optimal rates for synchronous average-reward $Q$-learning and asynchronous exponentially discounted $Q$-learning.

Our results assume tabular models, full communication, and honest worker measurements. Future work should relax these assumptions to handle function approximation, limited bandwidth, and adversarial data.

\appendix
\section*{Appendix}

%%%%%%%%%%%%%%%%%%%%%%%%%%%%%%%%%%%%%%%%%%%%%%%%

% \input{AAAI_2025_FedQ/ReproducibilityChecklist}

% \clearpage

\begin{table*}[ht]
    \centering
    \caption{Table of constants.}
    \begin{tabular}{|c|c|c|c|}
        \hline
        \textbf{Constants} & \textbf{Values} & \textbf{Constants} & \textbf{Values}
        \\
        \hline
        \hline 
        $\Cl$ &  $C^\textnormal{noise}_1$
        &
        $\Cq$     &   $\tstr\CD + \frac{\pi^2\Cgm C_G}{6} + C_2^\textnormal{noise} + \Cq_\textnormal{nonlin}$
        \\[1ex]
        $\Cq_\textnormal{nonlin}$ & $ C_\textnormal{mix}\Big(\frac{2c_u\KG\max\{\tstr\CD^2, C_Q\}}{c_\ell(1-\alpha)}\Big)\Big[\frac{\Cb}{\cstr^2} + \frac{\CA}{\cstr} \Big]$ & $\CD$ & $\left[ \p{\Delta_0} + 4\tstr(\Cb+\CA\p{\QS}) \right]\exp{(\tstr)}$
        \\[1ex]
        $\beta_\kappa$  & $\kappa(1-\beta) - \frac{\CE\CA}{(1-\rho)},$  & $\kappa$ & a fixed integer s.t. $\beta_\kappa>0$
        \\[1ex]
        $\Cgm$ & $(1-\beta)\CA \kappa^2\exp{((1+\CA)^\kappa}C_\alpha)$ & $\KG$ & $\Cgm\bigg[1 + \frac{e^{\alpha}}{1 - \alpha}[ \frac{2^{\alpha/(1 - \alpha)}  \kappa(1 - \alpha)}{\beta_\kappa} + K_\Gamma^{(2)} ] \bigg]$       \\[1ex]
        $\xi_\Gamma$ & $\frac{\pi^2 \Cgm}{6} \left(\frac{2  \kappa}{e \beta_\kappa} \exp\left(\frac{\beta_\kappa}{2 \kappa}\right) \right)^{2/(1 - \alpha)}$ & $K_\Gamma^{(2)}$ &  $\frac{1}{4} \left(\frac{4 - 2\alpha}{1 - \alpha} \right)^{\alpha/(1 - \alpha) + 2} e^{-\beta_\kappa (2 - \alpha)/\kappa(1 - \alpha)^2}$ \\[1ex]
        $C_1^\textnormal{noise}$ & $ 8\KG(\Cb + \CA\p{\QS}) $ & $C_2^\textnormal{noise}$ & $C_{2,1}^{\textnormal{noise}} + C^\textnormal{noise}_{2,2}$
        \\[1ex]
        $C_{2,1}^{\textnormal{noise}}$ & $(4\Cgm C_\alpha C_G  + 2\KG C_{\textnormal{E}, \rho}(\Cb + \CA\p{\QS})$ & $C_G$ & $\Big[ \frac{2e^{\left(\beta_\kappa/2\kappa\right)\tstr^{(1-\beta)}}}{e\beta_\kappa}\Big]^{\frac{2}{(1-\alpha)}}$
        \\[1ex]
        $C_\alpha$ & $\sum_{t=\tstr}^{\infty}\at^2$ & $C_{\textnormal{E},\rho}$ & $\CE\Big( \frac{\pi}{6} + \frac{2}{(1-\rho)} + 2\sqrt{\frac{C_\alpha}{1-\rho^2}} \Big)$
        \\[1ex]
        $C^\textnormal{noise}_{2,2}$ & $\Big(1 + \frac{4\kappa C_\Gamma}{\beta_\kappa} \Big)\Big( C^\textnormal{noise}_{2,1,1} + \frac{C_\textnormal{mix}}{(1-\alpha)}C^\textnormal{noise}_{2,1,2} \Big)$ & $C^\textnormal{noise}_{2,1,1}$ & $C_1^\textnormal{noise}(C_S+\Cgm^2) \Big[\frac{e^{\beta_\kappa\tstr/2\kappa}}{e\beta_\kappa}\Big]^{\frac{2}{(1-\beta)}}$
        \\[1ex]
        $C_S$ & $4\Cgm e^{\beta_\kappa/2\kappa}\Big( 1+ \Cgm\sqrt{\frac{2\CE}{(1-\rho)}}\Big)$ & $C^\textnormal{noise}_{2,1,2}$ & $ C_K\Big[\frac{2\KG}{\Cgm} + \frac{\KG}{(1-\beta)\Cgm} + \frac{\KG}{\Cgm^2} + \big(\frac{8}{\beta_\kappa} + \frac{C_G}{2} \big) \Big]$
        \\[1ex]
        $C_K$ & $2\Cgm C_S(\Cb + \CA\p{\QS})(1+\CE)$ & $C_\textnormal{mix}$ & $\Big[\frac{2\alpha}{\ln(1/\rho)} + \big|\frac{\ln \CE}{\ln(1/\rho)}\big|\Big]$
        \\
        \hline
    \end{tabular}
    \label{tab: constants}
\end{table*}

%%%%%%%%%%%%%%%%%%%%%%%%%%%%%%%%%%%%%%%%%%%%%%%%%%%%%%%%%%%%%%%%%%%%%%%%%%%%%%%%%%%%%%%%%%%%%%%%%%%%%%%%

\section{Proof of Theorem \ref{thm: SA.convergence}}
\label{appendix: SA.convergence}

Here, we provide the detailed proof of Theorem \ref{thm: SA.convergence}. We first recall a few concepts. Using the definitions of $b_i$ and $A_i$ from Assumption \ref{a: SA.update.function}, recall that for all $Q\in \bR^d$ and $t\geq 0,$
\[
    \bht^Q = \frac{1}{N}\sum_{i=1}^N b_i(Q,\yti), \qquad b^Q  = \frac{1}{N}\sum_{i=1}^N \bE_{y\sim\eta_i}[b_i(Q,y)],
\]
and 
\[
    \Aht^Q = \frac{1}{N}\sum_{i=1}^N A_i(Q,\yti), \quad A^Q  = \frac{1}{N}\sum_{i=1}^N \bE_{y\sim\eta_i}[A_i(Q,y)].
\]
With these definitions, the update rule \eqref{e: SA.main.update} can be written in the form
\begin{equation}\label{e: SA.main.update.compact}
    Q_{t+1} =(1-\at) Q_t + \at\left[ \bht^{Q_t} + \Aht^{Q_t}Q_t \right].
\end{equation}
From \eqref{e: fixed.point}, the fixed point $\QS$ of $H$ satisfies
\begin{equation}\label{e: SA.fixed.pt}
    b^{\QS} + A^{\QS}\QS = \QS + \eS,
\end{equation}
where $\eS\in E.$ Also, the raw-iterate error and its Polyak-Ruppert average satisfy
\[
    \Delta_t = Q_t - (\QS+\eS) \quad\text{and}\quad \bDelta_t = \bQ_t - (\QS+\eS).
\]
Finally, note that, for $T\geq 0,$
\[
    \bDelta_T = \frac{1}{T}\sum_{t=0}^{T-1}\Delta_t \quad  \text{and} \quad     \p{\bDelta_T} = \p{\bQ_T - \QS}.
\]

We now provide the details of the four steps, mentioned in Section~\ref{s: proof.main}, to prove Theorem \ref{thm: SA.convergence}.

\paragraph{Step 1 (Raw Error Decomposition).}

The first step is to decompose the raw iterate error $\Delta_t$ into a linear recursive iteration and a nonlinear perturbation, as summarized in Lemma~\ref{lem:raw.iterate.error.split}.

\begin{proof}[\textbf{Proof of Lemma \ref{lem:raw.iterate.error.split}}]
Define for all $t\geq 0,$ the noise term $\omega_t$ and nonlinear perturbation $\xi_t$ as
\begin{align*}
    \omega_t & := \Big[\bht^{\QS} - b^{\QS}\Big] + \left[\Aht^{\QS} - A^{\QS}\right]\QS + \eS,
    \\
    \xi_t &:= \Big[\bht^{Q_t} - \bht^{\QS}\Big] + \left[\Aht^{Q_t} - \Aht^{\QS}\right]Q_t.
\end{align*}
With this, we can rewrite \eqref{e: SA.main.update.compact}. Specifically, subtracting $(\QS+\eS)$ from both sides of \eqref{e: SA.main.update.compact}, we have
\begin{align*}
    \Delta_{t+1} &
    \\
    = & (1-\at)\Delta_t + \at[ \bht^{Q_t} + \Aht^{Q_t}Q_t - (\QS + \eS)]
    \\
    \overset{(a)}{=} &  (1-\at)\Delta_t  + \at[ \bht^{Q_t} - b^{\QS}] + \at[ \Aht^{Q_t}Q_t - A^{\QS}\QS]
    \\
    \overset{(b)}{=} & (1-\at)\Delta_t 
    + \at[ \bht^{\QS} - b^{\QS}] 
    \\ 
    & + \at[\Aht^{\QS}Q_t  - A^{\QS}\QS] + \at\xi_t
    \\
    \overset{(c)}{=} & (1-\at)\Delta_t + \at[ \bht^{\QS} - b^{\QS}]
    \\
    & + \at[ \Aht^{\QS}(\Delta_t+\QS+\eS) - A^{\QS}\QS] + \at\xi_t
    \\
    \overset{(d)}{=} & (1-\at)\Delta_t + \at\Aht^{\QS}\Delta_t + \at[ \bht^{\QS} - b^{\QS}]
    \\
    & + \at[ \Aht^{\QS}\QS  + \eS - A^{\QS}\QS] + \at\xi_t
    \\
    \overset{(e)}{=}& [\bI-\at [\bI - \Aht^{\QS}]]\Delta_t + \at\omega_t
    + \at\xi_t,
\end{align*}
where $(a)$ follows from \eqref{e: SA.fixed.pt}, $(b)$ follows from the definition of $\xi_t,$ $(c)$ follows from the definition of $\Delta_t,$ and $(d)$ follows from assumption \ref{a: SA.update.function}.(a) and the fact that $\eS\in E,$ while $(e)$ follows from the definition of $\omega_t.$ 

The relation in $(e)$ proves Lemma~\ref{lem:raw.iterate.error.split}, as desired
\end{proof}

\paragraph{Step 2 (Polyak-Ruppert Error Decomposition).} 

In this step, we use the decomposition of the raw iterate error obtained in \textbf{Step 1}, to obtain a similar decomposition of the Polyak-Ruppert averaged error $\bDelta_T$ into linear and nonlinear components. This decomposition is provided in Lemma~\ref{lem: PR.avg.error.split}, which we now prove.

\begin{proof}[\textbf{Proof of Lemma \ref{lem: PR.avg.error.split}}]

Define for all $0\leq t_1<t_2,$
\[
    \Gamma_{t_1:t_2} := \prod_{t=t_1}^{t_2-1}\left[\bI-\at\left[\bI - \Aht^{\QS}\right]\right],
\]
and for all $T>\tstr,$
\begin{align*}
        \bDeltrans_T &:= \bigg(  \frac{1}{T}\sum\limits_{t=0}^{\tstr-1}  \Delta_t \bigg),\ 
        \bDelnoise_T := \frac{1}{T}\sum\limits_{t=\tstr}^{T-1} \sum_{k=0}^{t-1}\ak\Gamma_{k+1:t} \omega_k, 
        \\
        \bDelinit_T & := \frac{1}{T}\sum\limits_{t=\tstr}^{T-1} \Gamma_{0:t} \Delta_0, 
        \  \bDelnlin_T := \frac{1}{T}\sum\limits_{t=\tstr}^{T-1} \sum_{k=0}^{t-1}\ak\Gamma_{k+1:t}\xi_k. 
\end{align*}
Recursively iterating the update rule \eqref{e: SA.update.modified},  we have for $t\geq 0,$
\begin{equation}\label{e: raw.error.compact}
\Delta_t = \Gamma_{0:t}\Delta_0 + \sum_{k=0}^{t-1}\ak\Gamma_{k+1:t}\ \omega_k  + \sum_{k=0}^{t-1}\ak\Gamma_{k+1:t}\ \xi_k.
\end{equation}
Hence, for all $T>\tstr,$ we have
\begin{align*}
    \bDelta_T & = \frac{1}{T}\sum_{t = 0}^{T-1}\Delta_t \\
    & = \frac{1}{T}\sum_{t=0}^{\tstr-1}\Delta_t + \frac{1}{T}\sum_{t=\tstr}^{T-1}\Delta_t
    \\
    & \overset{(a)}{=} \frac{1}{T}\sum_{t=0}^{\tstr-1}\Delta_t + \frac{1}{T}\sum_{t=\tstr}^{T-1}\Gamma_{0:t}\Delta_0 
    \\
    & + \frac{1}{T}\sum_{t=\tstr}^{T-1}\sum_{k=0}^{t-1}\ak\Gamma_{k+1:t}\omega_k  + \frac{1}{T}\sum_{t=\tstr}^{T-1}\sum_{k=0}^{t-1}\ak\Gamma_{k+1:t}\xi_k,
\end{align*}
where $(a)$ follows from \eqref{e: raw.error.compact}. Following the definitions of $\bDeltrans_T, \bDelinit_T, \bDelnoise_T,$ and $\bDelnlin_T,$ we have the desired decomposition of $\bDelta_T$ as
\[
    \bDelta_T = \bDeltrans_T + \bDelinit_T + \bDelnoise_T + \bDelnlin_T.
\]
Finally, taking the semi-norm on both sides and using the semi-norm triangle inequality, we have 
\[
    \bE \p{\bDelta_T} \leq \bAtrans_T + \bAinit_T + \bAnoise_T + \bAnlin_T.
\]
This concludes the proof of Lemma~\ref{lem: PR.avg.error.split}.
\end{proof}

\paragraph{Step 3 (Bounding the Averaged Linear Components).} 

In this step, we bound the error terms $\bAtrans_T,$ $\bAinit_T,$ and $\bAnoise_T,$ resulting from the linear component in the decomposition of $\Delta_t.$ Lemma~\ref{thm: PR.trans.bound} bounds the transient error term $\bAtrans_T,$ which results from the first $\tstr$ iterations of \eqref{e: SA.update.modified}. 

\begin{proof}[\textbf{Proof of Lemma \ref{thm: PR.trans.bound}}]
We begin by making the following claim on the iterates $\Delta_0,\ldots,\Delta_{\tstr-1}$.

\begin{claim}\label{claim: burning.period.bound}
    There exists a constant $\CD>0$ such that, for every $t<\tstr,$ we have $\p{\Delta_t}\leq \CD.$
\end{claim}

To see this, recall the update rule \eqref{e: SA.update.modified}, which states that
\[
    \Delta_{t+1}  = [\bI - \at[\bI - \Aht^{\QS}]]\Delta_t + \at\omega_t
    + \at\xi_t.
\]
Taking the semi-norm on the above expression and using the semi-norm triangle inequality, we have that
\begin{multline*}
    \p{\Delta_{t+1}} \leq \big[1+\at(1 + \p{\Aht^{\QS}}\big] \p{\Delta_t} \\
    + \at\p{\omega_t}
    + \at\p{\xi_t}
\end{multline*}
for all $t\geq 0.$ Form the definition of $\omega_t$ and $\xi_t,$ and Assumption \ref{a: SA.update.function}.(c), we have
\[
    \p{\omega_t} \leq 2\Cb + 2\CA\p{\QS},
\]
and
\begin{align*}
    \p{\xi_t} &\leq 2\Cb + 2\CA\p{Q_t}
    \\
    &\leq 2\Cb + 2\CA\p{\QS} + 2\CA\p{\Delta_t}.
\end{align*}
Hence,
\begin{multline*}
    \p{\Delta_{t+1}}\leq \p{\Delta_t} 
    \\
    + 4\at\left(\Cb + \CA\p{\QS} \right)
    + (1+3\CA) \at\p{\Delta_t}.
\end{multline*}
For all $t<\tstr,$ the above relation recursively shows that 
\begin{multline*}
    \p{\Delta_t} \leq \p{\Delta_0} 
    \\
    + 4(\Cb + \CA\p{\QS})\tstr 
    + (1+3\CA)\sum_{k=0}^{t-1}\ak\p{\Delta_k}.
\end{multline*}
Applying the discrete Gronwall's inequality now shows
\begin{align*}
    \p{\Delta_t} \leq {} & \big[\p{\Delta_0} + 4\tstr(\Cb+\CA\p{\QS})\big] \\
    & \times \exp\left( [1 + 3 C_A]  \sum_{k = 0}^{t - 1} \alpha_k\right)\\
    \overset{(b)}{\leq} {} & \big[\p{\Delta_0} + 4\tstr(\Cb+\CA\p{\QS})\big] \\
    & \times \exp(  [1 + 3 C_A] \tstr),
\end{align*}
where $(b)$ follows since $\alpha_k \leq 1$ for $k \geq 0$ and $t < \tstr.$
The desired claim now follows if we define
\begin{multline*}
    \CD := \big[ \p{\Delta_0} + 4\tstr(\Cb+\CA\p{\QS}) \big] \exp(  [1 + 3 C_A] \tstr).
\end{multline*}

Finally, Claim \ref{claim: burning.period.bound} allows us to bound $\bAtrans_T$ as follows:
\[
    \bAtrans_T \leq \frac{1}{T}\sum_{t=0}^{\tstr-1}\bE\p{\Delta_t} \leq \frac{\tstr\CD}{T}, \quad \forall T>\tstr.
\]
This concludes the proof of Lemma~\ref{thm: PR.trans.bound}.
\end{proof}

Next, Lemma~\ref{thm: SA.bound.matrix} bounds the conditional expectation of the semi-norm of the random matrix product, i.e., $\bE_{t_1}\p{\Gamma_{t_1:t_2}}$ with an exponentially decaying term.

\begin{proof}[\textbf{Proof of Lemma \ref{thm: SA.bound.matrix}}] The proof relies on the semi-norm contractive nature of the matrix $A^{\QS}$ and closely follows the recipe in \cite[Proposition 7]{durmus2025finite}. The proof sketch is laid out in Section~\ref{s: proof.main}.
\end{proof}

Next, we move to Lemma~\ref{thm: PR.init.bound}, which provides a bound on the initial value-based error term $\bAinit_T.$
\begin{proof}[\textbf{Proof of Lemma \ref{thm: PR.init.bound}}]
    Taking the semi-norm on $\bDelinit_T$  and using the semi-norm triangle inequality, we obtain
    \begin{align*}
        \bAinit_T \leq & \frac{1}{T}\sum_{t=\tstr}^{T-1}\bE\p{\Gamma_{0:t}}\p{\Delta_0}
        \\
        \overset{(a)}{\leq} & \frac{\Cgm \p{\Delta_0}}{T} \sum_{t = \tstr}^{T - 1} \exp \bigg(-\beta_\kappa \sum_{\ell = 0}^{\lfloor t/\kappa\rfloor } \alpha_{\ell \kappa}\bigg) \\
        \overset{(b)}{=} & \frac{\Cgm \p{\Delta_0}}{T} \sum_{t = \tstr}^{T - 1} \exp\bigg(-\beta_\kappa \sum_{\ell = 0}^{\lfloor t/\kappa\rfloor } \frac{1}{(\ell \kappa + 1)^\alpha}\bigg) \\
        \overset{(c)}{\leq} &  \frac{\Cgm \p{\Delta_0}}{T} \sum_{t = \tstr}^{T - 1} \exp\bigg(-\frac{\beta_\kappa}{\kappa(1 - \alpha)} [(t + 1)^{1 - \alpha} - 1]\bigg)
        \\
        \overset{(d)}{\leq} & \frac{\Cgm\p{\Delta_0}}{T} \sum_{t=\tstr}^{T-1} \frac{1}{(t + 1)^2} \left(\frac{2  \kappa}{e \beta_\kappa} \exp\left(\frac{\beta_\kappa}{2 \kappa}\right) \right)^{2/(1 - \alpha)}
        \\
        \overset{(e)}{\leq} & \frac{\pi^2 \Cgm}{6} \left(\frac{2  \kappa}{e \beta_\kappa} \exp\left(\frac{\beta_\kappa}{2 \kappa}\right) \right)^{2/(1 - \alpha)} \frac{\p{\Delta_0}}{T}
    \end{align*}
    where $(a)$ follows from Lemma~\ref{thm: SA.bound.matrix}, $(b)$ follows since $\alpha_{t} = 1/(t + 1)^\alpha,$ $(c)$ holds since 
    \[
        \sum_{\ell = 0} ^{L} (\ell \kappa + 1)^{-\alpha} \geq \frac{1}{\kappa (1 - \alpha)} \big[[(L + 1)\kappa + 1]^{1 - \alpha} - 1 \big]
    \]
    and $\lfloor t/\kappa \rfloor + 1 \geq t/\kappa,$ $(d)$ holds since 
    \[
        x^2\cdot \exp\bigg(-\frac{\beta_\kappa}{\kappa(1-\alpha)}x^{(1-\alpha)}\bigg) \leq \bigg(\frac{2\kappa}{e\beta_\kappa}\bigg)^{2/(1-\alpha)},
    \]
    while $(e)$ follows from the fact that $\sum_{t=1}^\infty \frac{1}{t^2}=\pi^2/6$. This completes the proof of Lemma~\ref{thm: PR.init.bound}.
\end{proof}

%
% \begin{proof}[\textbf{Proof of Lemma \ref{thm: PR.init.bound}}]
%     %
%     Taking the semi-norm on $\bDelinit_T$  and using the semi-norm triangle inequality, we obtain
%     %
%     \begin{align*}
%         %
%         \bAinit_T \leq & \frac{1}{T}\sum_{t=\tstr}^{T-1}\bE\p{\Gamma_{0:t}}\p{\Delta_0}
%         \\
%         \overset{(a)}{\leq} & \frac{\Cgm\p{\Delta_0}}{T}\sum_{t=\tstr}^{T-1}e^{-\frac{\beta_\kappa}{\kappa (1-\alpha)}[t^{(1-\alpha)} - \tstr^{(1-\beta)} ]}
%         \\
%         \overset{(b)}{\leq} & \frac{\Cgm\p{\Delta_0}}{T}\Big[e^{\frac{\beta_\kappa \tstr^{(1-\beta)}}{\kappa(1-\alpha)}}\Big]\sum_{t=\tstr}^{T-1}\frac{1}{t^2}\Big( \frac{2\kappa}{e\beta_\kappa} \Big)^{\frac{2}{(1-\alpha)}}
%         \\
%         \overset{(c)}{\leq} & \bigg[\frac{\pi^2\Cgm}{6}\Big[e^{\frac{\beta_\kappa \tstr^{(1-\beta)}}{\kappa(1-\alpha)}}\Big]\Big( \frac{2\kappa}{e\beta_\kappa} \Big)^{\frac{2}{(1-\alpha)}}\bigg] \frac{\p{\Delta_0}}{T},
%         %
%     \end{align*}
%     %
%     where $(a)$ follows from Lemma~\ref{thm: SA.bound.matrix}, $(b)$ follows from the fact that $x^2\cdot e^{-\frac{\beta_\kappa}{\kappa(1-\alpha)}x^{(1-\alpha)}}\leq \big(2\kappa/e\beta_\kappa\big)^{2/(1-\alpha)},$ and $(c)$ follows from the fact that $\sum_{t=1}^\infty \frac{1}{t^2}=\pi^2/6$. This completes the proof of Lemma~\ref{thm: PR.init.bound}.
%     %
% \end{proof}
%
Lemma~\ref{thm: PR.noise.bound} bounds the error term $\bAnoise_T$ caused by the noise $\omega_t$ in \eqref{e: SA.update.modified}.
\begin{proof}[\textbf{Proof of Lemma \ref{thm: PR.noise.bound}}]
The desired result  follows by mimicking Lemma~IV.7's proof from \cite{naskar2025}. 
\end{proof}
\paragraph{Step 4 (Analysis of the Nonlinear term).} In this step, we get a bound on $\bAnlin_T.$ We begin with Lemma~\ref{e: matrix.uniform.bound.}'s proof. 

\begin{proof}[\textbf{Proof of Lemma \ref{e: matrix.uniform.bound.}}]
    We have
    \begin{align*}
        \bE_{t_1} & \p{M_{t_1}^{t_2}}\\  
        \overset{(a)}{\leq} {}  & \alpha_{t_1} \sum_{s = t_1 + 1}^{t_2 - 1} \p{\Gamma_{t_1 + 1:s}}
        \\
        \overset{(b)}{\leq} {} & \Cgm \alpha_{t_1} \sum_{s = t_1 + 1}^{t_2 - 1} \exp \bigg(-\beta_\kappa \sum_{\ell = 0 }^{\lfloor (s - t_1 - 1)/\kappa \rfloor} \alpha_{t_1 + 1 + \ell \kappa} \bigg) \\
        \overset{(c)}{=} &  \Cgm \alpha_{t_1}  \sum_{s = t_1 + 1}^{t_2 - 1} \exp\bigg(-\beta_\kappa \sum_{\ell = 0}^{\lfloor (s - t_1 - 1)/\kappa \rfloor} \frac{1}{(\ell \kappa + t_1 + 2)^\alpha}\bigg) \\
        \overset{(c)}{\leq} & \Cgm \alpha_{t_1} \sum_{s = t_1 + 1}^{t_2 - 1} \exp\bigg(-\frac{\beta_\kappa [(s + 1)^{1 - \alpha} - (t_1 +2)^{1 - \alpha}]}{\kappa(1 - \alpha)} \bigg)
    \end{align*}
    where $(a)$ follows from triangle inequality, $(b)$ from Lemma~\ref{thm: SA.bound.matrix}, $(c)$ follows since $\alpha_{t} = 1/(t + 1)^\alpha,$ while $(d)$ follows by interpreting the given expression as a Riemann sum approximating $\int_{\ell = 0}^{\lfloor (s - t_1 - 1)/\kappa\rfloor + 1} (x \kappa + t_1 + 2)^{-\alpha} \textnormal{d}x.$ 

    Now, let 
    \[
        I := \sum_{s = t_1 + 1}^{t_2 - 1} \exp\bigg[-\frac{\beta_\kappa [(s + 1)^{1 - \alpha} - (t_1 +2)^{1 - \alpha}]}{\kappa(1 - \alpha)} \bigg]
    \]
    so that
    \begin{equation}
    \label{e:M.Bd}
        \bE_{t_1} \p{M_{t_1}^{t_2}} \leq \Cgm \alpha_{t_1} I.
    \end{equation}
    Interpreting $I$ as a Riemann-sum approximation then shows
    \[
        I \leq 1 + \int_{t_1 + 2}^\infty f(x) \textnormal{d}x,
    \]
    where $f(x) = e^{- c (x^{1- \alpha} - \zeta^{1- \alpha})},$ $c = \beta_\kappa/(\kappa( 1- \alpha)),$ and $\zeta = t_1 + 2.$ Next let $u = x^{1-  \alpha} - \zeta^{1 - \alpha}.$ Then, 
    \[
        x = [u + \zeta^{1- \alpha}]^{1/(1 - \alpha)}
    \]
    and, hence, 
    \[
        \textnormal{d}x = \frac{1}{1 - \alpha} [u + \zeta^{1- \alpha}]^{\alpha/(1 - \alpha)}  \textnormal{d}u.
    \]
    Therefore, from \eqref{e:M.Bd},
    \begin{align}
        \bE_{t_1} & \p{M_{t_1}^{t_2}} \nonumber \\
        \leq {} & \Cgm \alpha_{t_1} \left[1 + \frac{1}{1 - \alpha} \int_{0}^\infty e^{-cu} [u + \zeta^{1- \alpha}]^{\alpha/(1 - \alpha)} \textnormal{d}u\right] \nonumber \\
        \leq {} & \Cgm \alpha_{t_1} \left[  1 + \frac{1}{1 - \alpha}[I_1 + I_2] \right], \label{e:I.decomposition}
    \end{align}
    where 
    \begin{align*}
        I_1 := {} & \int_{0}^{\zeta^{1 - \alpha}} e^{-cu} [u + \zeta^{1- \alpha}]^{\alpha/(1 - \alpha)}  \textnormal{d}u\\
        \intertext{and}
        I_2 := {} & \int_{\zeta^{1 - \alpha}}^{\infty} e^{-cu} [u + \zeta^{1- \alpha}]^{\alpha/(1 - \alpha)}   \textnormal{d}u.
    \end{align*}
    
    Now, for $0 \leq u  \leq \zeta^{1 - \alpha},$ we have 
    \[
        u + \zeta^{1 - \alpha} \leq 2\zeta^{1 - \alpha}.
    \]
    Hence, 
    \[
        I_1 \leq \frac{2^{\alpha/(1 - \alpha)}  \zeta^\alpha}{c} = \frac{2^{\alpha/(1 - \alpha)}  \kappa(1 - \alpha) \zeta^\alpha }{\beta_\kappa},
    \]
    where the last relation follows since $c = \beta_\kappa/(\kappa(1 - \alpha).$
    Similarly, 
    \begin{align*}
        I_2 \overset{(a)}{\leq} {} & 2^{\alpha/(1 - \alpha)} \int_{\zeta^{1 - \alpha}}^\infty e^{-cu}u^{\alpha/(1 - \alpha)}  \textnormal{d}u \\
        \overset{(b)}{\leq} {} & 2^{\alpha/(1 - \alpha)} \left(\frac{2- \alpha}{1 - \alpha} \right)^{\alpha/(1 - \alpha) + 2} e^{-c(2 - \alpha)/(1 - \alpha)}  \\
        {} & \times \int_{\zeta^{1 - \alpha}}^\infty u^{-2} \textnormal{d}u \\
        \leq {} & \frac{1}{4} \left(\frac{4 - 2\alpha}{1 - \alpha} \right)^{\alpha/(1 - \alpha) + 2}  \frac{e^{-c(2 - \alpha)/(1 - \alpha)}}{\zeta^{1 - \alpha}} \\
        \leq {} & \frac{1}{4} \left(\frac{4 - 2\alpha}{1 - \alpha} \right)^{\alpha/(1 - \alpha) + 2}  \frac{e^{-\beta_\kappa (2 - \alpha)/\kappa(1 - \alpha)^2}}{\zeta^{1 - \alpha}}
    \end{align*}
    where $(a)$  holds since $u + \zeta^{1 - \alpha} \leq 2u,$ $(b)$ holds since $e^{-cu}u^{\alpha/(1 - \alpha) + 2}$ is maximized at $u = [\alpha/(1 - \alpha) + 2]/c,$ while $(c)$ follows by substituting $c = \beta_\kappa/(\kappa(1 - \alpha)).$
    
    Substituting these bounds in \eqref{e:I.decomposition} and then making use of the facts that $\zeta \geq 1$ and
    \[
        \alpha_{t_1} \zeta^{\alpha} = \frac{1}{(t_1 + 1)^{\alpha}}(t_1 + 2)^{\alpha} \leq e^{\alpha},
    \]
    we get $I \leq \KG,$ as desired.    
\end{proof}

Next, we formally derive \eqref{e:xi_k.Bd}. 

\begin{lemma}\label{lem: nonlin.bound.appendix}
    For all $k\geq 0,$
    \[
        \p{\xi_k} \leq 2\Big(\frac{\CA}{\cstr} + \frac{\Cb}{\cstr^2}\Big)\p{Q_k - \QS}^2.
    \]
\end{lemma}

\begin{proof}
     From Assumption~\ref{a: SA.update.function}.(b) and the definition of $\xi_k,$ we have the following sandwiching relation:
    \begin{equation}\label{e: sandwich.xi}
        0 \leq \xi_k - [\bhk^{Q_k} - \bhk^{\QS}] \leq [\Ahk^{Q_k} - \Ahk^{\QS}][Q_k - \QS].
    \end{equation}
    Now, Lemma~\ref{lem:semi-norm.norm.coupling} gives a norm $\|\cdot\|$ such that
    \begin{equation}\label{e: semi.equal.norm}
        \p{Q_k - \QS} = \|Q_k - \QS - e_k\|,
    \end{equation}
    where $e_k := \argmin_{e\in E}\|Q_k - \QS - e\|.$ 

    Let $\|\cdot\|_m$ be a chosen monotone norm on $\bR^d,$ and let $c_u,c_\ell>0$ be constants which satisfy \eqref{e: span.semi.norm.monotone.induced.norm}. 
    Then, for all $k\geq 0,$ we have
    \begin{align}\label{e: xi.bound.split}
        &\p{\xi_k} 
        \nonumber\\
        & \overset{(a)}{\leq} \|\xi_k\| \\
        & \overset{(b)}{\leq} c_u\|\xi_k\|_m
        \nonumber\\
        & \overset{(c)}{\leq} c_u\|\bht^{Q_k} - \bhk^{\QS}\|_m + c_u\|\xi_k - [\bht^{Q_k} - \bhk^{\QS}]\|_m
        \nonumber\\
        & \overset{(d)}{\leq}  c_u\| \bht^{Q_k} - \bhk^{\QS} \|_m 
        + c_u\| [\Ahk^{Q_t} - \Ahk^{\QS}]  [Q_k - \QS]\|_m
        \nonumber \\
        & \overset{(e)}{\leq}  \frac{c_u}{c_\ell}\| \bht^{Q_k} - \bhk^{\QS} \| 
        + \frac{c_u}{c_\ell}\| [\Ahk^{Q_t} - \Ahk^{\QS}] [Q_k - \QS]\|
        \nonumber \\
        & \overset{(f)}{=}  \frac{c_u}{c_\ell}\| \bht^{Q_k} - \bhk^{\QS} \| 
        + \frac{c_u}{c_\ell}\| [\Ahk^{Q_t} - \Ahk^{\QS}] [Q_k - \QS - e_k]\|
        \nonumber\\
        & \overset{(g)}{=}  \frac{c_u}{c_\ell}\| \bht^{Q_k} - \bhk^{\QS} \| 
        + \frac{c_u}{c_\ell}\| \Ahk^{Q_t} - \Ahk^{\QS}\| \|Q_k - \QS - e_k\|
        \nonumber\\
        & \overset{(h)}{=}  \frac{c_u}{c_\ell} \left[\|\bht^{Q_k} - \bhk^{\QS} \| 
        + \| \Ahk^{Q_t} - \Ahk^{\QS}\|\ \p{Q_k -\QS} \right], \label{e:xi_k.Intermediate.Bd}
    \end{align}
    where $(a)$ follows from Lemma~\ref{lem:semi-norm.norm.coupling}, $(b)$ follows from \eqref{e: span.semi.norm.monotone.induced.norm}, $(c)$ follows from triangle inequality, $(d)$ follows from the monotonicity of $\|\cdot\|_m$ applied to \eqref{e: sandwich.xi}, $(e)$ follows from \eqref{e: span.semi.norm.monotone.induced.norm}, $(f)$ follows from Assumption \ref{a: SA.update.function}.(a) since $e_k\in E,$ $(g)$ uses the inequality $\|Bx\|\leq \|B\|\|x\|,$ and $(h)$ follows from \eqref{e: semi.equal.norm}.

    Next, from Assumptions~\ref{a: SA.update.function}.(b) and \ref{a: SA.update.function}.(d), we have
    \[
        \|\bhk^{Q_k} - \bhk^{\QS}\| \begin{cases} = 0 & \text{if } \p{Q_k - \QS}<\cstr, \\ \leq 2\Cb &  \text{otherwise,}  \end{cases}
    \]
    and
    \[
        \|\Ahk^{Q_k} - \Ahk^{\QS}\| \begin{cases} = 0 & \text{if } \p{Q_k - \QS}<\cstr, \\ \leq 2\CA &  \text{otherwise.}  \end{cases}
    \]
    Therefore,
    \begin{align}\label{e: norm.diff.b}
    \|\bhk^{Q_k} - \bhk^{\QS}\|  \leq & 2\Cb\ \ones_{\{\p{Q_t -\QS}\geq \cstr\}}
    \nonumber \\
    \leq  & (2\Cb/\cstr^2)\ \p{Q_k - \QS}^2,
    \end{align}
    and
    \begin{align}\label{e: norm.diff.A}
    \|\Ahk^{Q_k} - \Ahk^{\QS}\|  \leq & 2\CA\ \ones_{\{\p{Q_t -\QS}\geq \cstr\}}
    \nonumber \\
    \leq & (2\CA/\cstr)\ \p{Q_k - \QS}.
    \end{align}

    Combining \eqref{e: xi.bound.split}, \eqref{e:xi_k.Intermediate.Bd}, \eqref{e: norm.diff.b}, and \eqref{e: norm.diff.A} concludes the proof of Lemma~\ref{lem: nonlin.bound.appendix}.
\end{proof}

From Assumption~\ref{a: raw.iterate.convergence.rate} and Claim~\ref{claim: burning.period.bound}, we now get the following corollary to Lemma~\ref{lem: nonlin.bound.appendix}:
\begin{corollary}\label{corol: nonlin.bound.appendix}
    For all $k< \tstr,$ we have
    \[
        \p{\xi_k} \leq \Big(\frac{2c_u}{c_\ell}\Big)\Big(\frac{\CA}{\cstr} + \frac{\Cb}{\cstr^2}\Big)\CD^2
    \]
    For all $k\geq\tstr,$ we have
    \[
        \bE\p{\xi_k} \leq \Big(\frac{2c_u}{c_\ell}\Big)\Big(\frac{\CA}{\cstr} + \frac{\Cb}{\cstr^2}\Big)C_Q\tauk\ak.
    \]
\end{corollary}

We now prove Lemma~\ref{thm: PR.nlin.bound}.
 
\begin{proof}[\textbf{Proof of Lemma~\ref{thm: PR.nlin.bound}}]
From \eqref{e:xi_k.Intermediate.Bd}, we have
\begin{equation}
    \bAnlin_T \leq \frac{1}{T}\sum_{k=0}^{T-2} \bE\big[ \p{M_k^T}\ \p{\xi_k}\big].
\end{equation}
Now, 
\begin{align}
    \bE\big[ \p{M_k^T}\ \p{\xi_k}\big] = 
    {} & \bE\bE_k\big[ \p{M_k^T}\ \p{\xi_k}\big]
    \\
    \overset{(a)}{=} {} & \bE\big[ \big(\bE_k\p{M_k^T}\big)\ \p{\xi_k}\big]\\
    \overset{(b)}{\leq} {} & \KG\bE\p{\xi_k},
\end{align}
where $(a)$ follows since $\p{\xi_k}$ is measurable w.r.t $\cF_k,$ while $(b)$ follows from Lemma~\ref{e: matrix.uniform.bound.}. Consequently,
\begin{equation}
    \bAnlin_T \leq \frac{\KG}{T}\sum_{k=0}^{T-2} \bE\p{\xi_k}.
\end{equation}
From Corollary~\ref{corol: nonlin.bound.appendix}, it then  follows that 
\begin{align*}
    \bAnlin_T \leq  \frac{2c_u\KG }{c_\ell T}\Big(\frac{\CA}{\cstr} + \frac{\Cb}{\cstr^2}\Big)\bigg[\tstr\CD^2 + \sum_{t=\tstr}^{T-2}C_Q\taut\at\bigg].
\end{align*}
Using the fact that $\taut \leq \tauT$ for $t\leq T-2,$ we then have
\[
    \frac{1}{T}\sum_{t=\tstr}^{T-2}\taut\at \leq \frac{\tauT}{T}\sum_{t=\tstr}^{T-2}\at \leq \frac{\tauT\ T^{(1-\alpha)}}{(1-\alpha)T} = \frac{\tauT}{(1-\alpha)T^\alpha}.
\]
Separately, $1/T\leq \tauT/T^\alpha$ for all $T>\tstr.$ Therefore,
\[
    \bAnlin_T  
    \leq  \frac{2c_u\KG \max\{\tstr\CD^2, C_Q\}}{c_\ell(1-\alpha)}\Big(\frac{\CA}{\cstr} + \frac{\Cb}{\cstr^2}\Big)\Big(\frac{\tauT}{T^\alpha}\Big).
\]
This concludes the proof of Lemma~\ref{thm: PR.nlin.bound}.
\end{proof}

Finally, we prove Theorem~\ref{thm: SA.convergence}. 

\begin{proof}[\textbf{Proof of Theorem~\ref{thm: SA.convergence}}]
By combining Lemmas~\ref{lem: PR.avg.error.split}, \ref{thm: PR.trans.bound}, \ref{thm: PR.init.bound}, \ref{thm: PR.noise.bound}, and \ref{thm: PR.nlin.bound}, it follows that
\begin{align*}
    \bE\p{\bDelta_T} & \leq \frac{\tstr\CD}{T} + \frac{\xi_\Gamma \p{\Delta_0}}{T}\\
    & + \frac{C^\text{noise}_1\sqrt{\tauT}}{\sqrt{NT}} + \frac{C^\text{noise}_2\ln(T)}{T^\alpha} 
    \\
    & + \frac{2c_u\KG \max\{\tstr\CD^2,C_Q\}}{c_\ell(1-\alpha)}\Big( \frac{\Cb}{\cstr^2}+\frac{\CA}{\cstr}\Big)\left(\frac{\tauT}{T^\alpha}\right).
\end{align*}
Collecting like terms and using $1/T\leq 1/T^\alpha,$ we then get
\begin{align*}
    & \bE\p{\bDelta_T} \leq \frac{C^\text{noise}_1\sqrt{\tauT}}{\sqrt{NT}}  + \bigg[ C^\text{noise}_2\ln(T)
    \\
    & \qquad\qquad   + \tstr\CD + \xi_\Gamma \p{\Delta_0}
    \\
    & \qquad\qquad  + \frac{2c_u\KG \max\{\tstr\CD^2,C_Q\}}{c_\ell(1-\alpha)}\Big( \frac{\Cb}{\cstr^2}+\frac{\CA}{\cstr}\Big)\tauT\bigg]\frac{1}{T^\alpha}.
\end{align*}
Finally, using $\tauT\leq \Big(\frac{2\alpha}{\ln(1/\rho)} + \big|\frac{\ln \CE}{\ln(1/\rho)}\big|\Big)\ln(T)$ along with $1\leq \ln(T)$ completes the proof of Theorem~\ref{thm: SA.convergence}. 
\end{proof}

\section{Q-Learning Applications: Experiments and Proofs of Theorems \ref{thm: avgQ.convergence} and   \ref{thm: expQ.convergence}}

In this section, we prove Theorems~\ref{thm: avgQ.convergence} and \ref{thm: expQ.convergence} by verifying the assumptions \ref{a: SA.update.function}---\ref{a: raw.iterate.convergence.rate} of Theorem~\ref{thm: SA.convergence}, following the proof outline in Section~\ref{s: proof.sketch.Application.RL}. We also present experiments on synthetic examples showing that our parameter-free distributed $Q$-learning algorithms perform comparably to their parameter-dependent counterparts, while improving as the number of agents increases. 

\subsection{Average-reward Q-learning}
\label{appendix: proof.avgQ.converge}
Here we prove that  the synchronous $J$-step average-reward $Q$-learning algorithm (see Section~\ref{s: application.RL}) satisfies the sufficient conditions of Theorem~\ref{thm: SA.convergence}, namely Assumptions~\ref{a: SA.update.function}--\ref{a: raw.iterate.convergence.rate}. 

Recall that synchronous average-reward $Q$-learning can be viewed as a special case of the fixed-point update rule in~\eqref{e: SA.main.update}. In this formulation, $d = \SA$ and $\cY = \cS^{(\cS \times \cA) \times J}$. For each $i$, the process $(y_t^i)_{t \geq 0}$ satisfies $y_t^i(s,a,k) \sim \cP_i(\cdot \mid s,a)$, independently across $k$, $(s,a)$, $t$, and $i$.

Given $Q \in \bR^{\SA}$ and $y \in \cY$, the $J$-step rollout for a state-action pair $(s,a)$ is defined recursively as follows: set $s_0 = s$ and $a_0 = a$, and for $k = 1,\ldots,J$,
\[
s_k = y(s_{k-1}, \pQ(s_{k-1}), k),
\]
where $\pQ$ is the greedy policy with respect to $Q$.

With this notation in place, for $Q \in \bR^{|\cS||\cA|}$ and $y \in \cY$, the local driving function $h_i$ is given by
\begin{multline}\label{appendix: local.function}
    h_i(Q,y)(s,a) := \cR_i(s,a) 
    + \sum_{k=1}^{J-1}\cR_i(s_k^{(s,a)}, \pQ(s^{(s,a)}_k))
    \\
    + \max_{a'} Q(s_J^{(s,a)}, a').
\end{multline}

We now verify that this function $h_i$ and the associated Markov processes $(y_t^i)$ satisfy Assumptions~\ref{a: SA.update.function}--\ref{a: raw.iterate.convergence.rate}.

Condition~\ref{a: SA.noise} holds trivially, since for each agent $i$, the sequence $(y_t^i)_{t \geq 0}$ is IID. Moreover, assuming Condition~\ref{a: SA.update.function}, it follows from \cite[Eq.~B7, Proof of Theorem~2]{zhang2021finite} that Condition~\ref{a: raw.iterate.convergence.rate} is also satisfied. Therefore, it remains to only verify Assumption~\ref{a: SA.update.function}.

We begin by noting that, in the average-reward $Q$-learning setting, the semi-norm $\p{\cdot}$ corresponds to the span semi-norm $\|\cdot\|_\spf$ defined in~\eqref{e:span.semi.norm}, with the associated linear subspace $E$ given by the span of the all-ones vector $\ones$.

Next, we express the local function $h_i$ in the form required by Assumption~\ref{a: SA.update.function}, namely,
\[
    h_i(Q,y) = b_i(Q,y) + A_i(Q,y)Q,
\]
for $Q \in \bR^{\SA}$ and $y \in \cY$. To this end, we introduce a family of stochastic matrices $\{\hcP_k^{\pQ}(y)\}_{k=0}^J$. Given $Q \in \bR^{|\cS||\cA|}$, $y \in \cY$, an initial state-action pair $(s,a)$, and the associated $J$-step rollout $(s_1^{(s,a)}, \ldots, s_J^{(s,a)}),$ the matrix
\begin{equation}\label{e: stoch.matrix.avg}
    \big[\hcP_k^{\pQ}(y)\big](s,a) := \ones^\top_{(s_k^{(s,a)}, \pQ(s_k^{(s,a)}))}.
\end{equation}
Equivalently, for each $(s,a)$ and $(s',a')$,
\[
    \hcP_k^{\pQ}(y)(s,a,s',a') =
    \begin{cases}
        1, & \text{if } s' = s_k^{(s,a)} \text{ and } a' = \pQ(s'),\\
        0, & \text{otherwise.}
    \end{cases}
\]
With this notation, the function $h_i$ in~\eqref{appendix: local.function} can be written as
\[
    h_i(Q,y) = \sum_{k=0}^{J-1} \hcP_k^{\pQ}(y)\,\cR_i \;+\; \hcP_J^{\pQ}(y)\,Q.
\]
Therefore, defining
\[
    b_i(Q,y) := \sum_{k=0}^{J-1} \hcP_k^{\pQ}(y)\,\cR_i,
    \text{ and }
    A_i(Q,y) := \hcP_J^{\pQ}(y),
\]
yields the desired representation of $h_i$.

Now, we show that each $b_i$ and $A_i$ satisfies the conditions \ref{a: SA.update.function}.(a)--\ref{a: SA.update.function}.(d).  Condition~\ref{a: SA.update.function}.(a) follows since $\hcP^{\pQ}_{J}(y)$ is row-stochastic and $E$ is spanned by $\ones.$ Condition~\ref{a: SA.update.function}.(b) follows since $\pQ$ is greedy w.r.t $Q,$ i.e., for each state $s',$ and action $a',$ we have $Q(s',\pQ(s'))\geq Q(s',a').$ 

For condition~\ref{a: SA.update.function}.(c), we assume that the local reward functions are bounded uniformly, i.e., $|\cR_i(s,a)|\leq \Rm$ for each agent $i,$ each state-action pair $(s,a),$ and some fixed $\Rm>0.$ Then, we have
\begin{align*}
    \| b_i(Q,y)\|_\spf &= \Big\|\sum_{k=0}^{J-1}\hcP_{k}^{\pQ}(y)\cR_i \Big\|_\spf 
    \\
    & \overset{(a)}{\leq} \sum_{k=0}^{J-1}\|\hcP_{k}^{\pQ}(y)\cR_i\|_\spf
    \overset{(b)}{\leq} 2\sum_{k=0}^{J-1}\|\hcP_{k}^{\pQ}(y)\cR_i\|_\infty 
    \\
    & \overset{(c)}{\leq} 2\sum_{k=0}^{J-1}\|\cR_i\|_\infty \leq 2J\Rm,
\end{align*}
where $(a)$ uses the semi-norm triangle inequality, $(b)$ follows as $\|\cdot\|_\spf \leq 2\|\cdot\|_\infty,$ and $(c)$ follows since for any row-stochastic matrix $P,$ we have $\|P\|_\infty\leq 1.$ From a similar argument, we have $\|A_i(Q,y)\|_\spf\leq 1.$ Therefore, letting $\Cb:= 2J\Rm$ and $\CA:=1$ suffices for condition~\ref{a: raw.iterate.convergence.rate}.(c).

At last, we verify condition~\ref{a: SA.update.function}.(d). We need to show that there exists $\cstr>0$ such that, for each agent $i,$ $y\in \cY,$ and $Q\in \bR^{\SA},$ we have $A_i(Q,y)=A_i(\QS,y)$ whenever $\|Q-\QS\|_\spf<\cstr.$ We exploit the fact that $\pQ$ is piecewise constant on $\bR^{\SA}$ to partition this space into finitely many conical regions. Each cone $\cC_{\ba}$ is defined w.r.t a unique assignment $\ba\in \cA^\cS$ of actions to states under the greedy policy. In other words, for each $\ba \in \cA^{\cS},$ we define a cone 
\[
    \cC_{\ba} := \left\{ Q\in \bR^{\SA} : \ba(s) = \pQ(s), \text{ for all } s\in\cS \right\}.
\]
It follows that $\{\cC_{\ba}: \ba \in \cA^{\cS}\}$ partitions $\bR^{\SA}$ such that, $\pQ = \ba$ if and only if $Q\in \cC_{\ba}.$ We further assume that $\QS$ does not lie on the boundary of any cone. The desired constant $\cstr$ in \ref{a: SA.update.function}.(d) then depends on how far $\QS$ sits from the boundary of its corresponding cone. 

\noindent We begin by defining the spaces 
\[
    E := \{ c\ones: c\in \bR\} \quad \text{and} \quad E^\perp := \{ Q\in \bR^{\SA} : Q\perp E \}.
\]
Given any $c>0,$ we define the disc of radius $c$ which is centered at $\QS$ and confined to the hyperplane $\QS+E^\perp$ as follows:
\[
    D(\QS, c) := \{ Q \in (\QS+ E^\perp) : \|Q-\QS\|_\infty <c \}.
\]
We make the following claim:
\begin{claim}
    For any $c>0,$ we have $Q\in E + D(\QS,c)$ whenever $\|Q-\QS\|_\spf<c.$
\end{claim}

To see this, note that for every $e\in E$ and $e^\perp\in E^\perp,$ we have $\|e\|_\spf=0$ and $\|e^\perp\|_\infty \leq\|e^\perp\|_\spf\leq 2\|e^\perp\|_\infty.$ Since $E + E^\perp$ forms a direct sum decomposition of $\bR^{\SA},$ given any $Q\in \bR^{\SA},$ we can write $Q-\QS=e+e^\perp,$ for some $e\in E$ and $e^\perp\in E^\perp.$ If $\|Q-\QS\|_\spf<c,$ then
\[
    \|e^\perp\|_\infty \leq \|e^\perp\|_\spf = \|Q-\QS\|_\spf <c.
\]
Hence $\QS+e^\perp \in D(\QS,c)$ and $Q\in E + D(\QS,c).$

Now, we show that $\exists \cstr>0$ such that $D(\QS, \cstr)$ lies in the interior of the cone $\cC_{\bar{a}^*}.$ Since adding $c\ones$ to any vector does not affect its cone, the cylinder $E + D(\QS,\cstr)$ also lies in the interior of $\cC_{\bar{a}^*}.$ Thus, whenever $\|Q-\QS\|_\spf< \cstr,$ we have $Q\in E + D(\QS,\cstr)$ lying in the interior of $\cC_{\bar{a}^*}.$ Consequently, $Q$ and $\QS$ lie in the same cone; and hence, $\pQ=\pi_{\QS}.$ 

Since $\QS$ lies in the interior of $\cC_{\bar{a}^*},$ there exists an open ball 
\[
    B(\QS, \cstr) := \left\{ Q\in \bR^{\SA} : \|Q-\QS\|_\infty < \cstr \right\}
\]
which is centered at $\QS$ and sits in the interior of $\cC_{\bar{a}^*}.$ Taking intersection of this ball with the hyperplane $(\QS+E^\perp)$ produces this desired disc $D(\QS, \cstr).$ 

This completes the verification of condition~\ref{a: SA.update.function}.(d) and hence proves Theorem~\ref{thm: avgQ.convergence}. \qed

% %%%%%%%%%%%%%%%%%%%%%%%%%%%%%%%%%%%%%%%%%%%%%%%%%%%%%%%%%%%%%%%%%%%%%%%%%%%%%%%%%%%%%%%%%%%%%%%%%%%%%%%%
\subsection{Asynchronous Discounted Q-learning}
\label{appendix: application.exp}

In this section, we prove Theorem~\ref{thm: expQ.convergence}. We need to show that conditions~\ref{a: SA.update.function}--\ref{a: raw.iterate.convergence.rate} hold for the asynchronous exponentially discounted $Q$-learning algorithm. 

Recall from Section~\ref{s: application.RL}, that this algorithm is a special case of \eqref{e: SA.main.update}---with $d=\SA,$ $ \cY= \cS\times\cA\times\cS,$ and the local function $h_i:\bR^{\SA}\times\cY\to\bR^{\SA}$ defined as 
\begin{multline}\label{appendix: local.function.exp}
    h_i(Q,s,a,s') := Q - Q(s,a)\ones_{(s,a)} 
    \\
    + [\cR_i(s,a) + \gamma \max_{a'} Q(s',a') ]\ones_{(s,a)},
\end{multline}
for each agent $i,$ $Q\in \bR^{\SA},$ and $(s,a,s')\in \cY$---while the local Markov chain $(\yti)$ is defined as $\yti=(s^i_t,a^i_t,s^i_{t+1}),$ where $a^i_t\sim \mu(\cdot|s^i_t),$ and $s^i_{t+1}\sim \cP_i(\cdot| s^i_t, a^i_t),$ for all $t\geq 0.$ 

Condition~\ref{a: SA.noise} states that the Markov chain $(\yti)$ satisfies a geometric mixing property, which is a standard assumption (see discussion below \cite[Assumption 3.1]{chen2021lyapunov}), while condition~\ref{a: raw.iterate.convergence.rate} follows from \cite[Theorem B.1]{chen2021lyapunov}. Therefore, we only need to verify condition~\ref{a: SA.update.function} to prove of Theorem~\ref{thm: expQ.convergence}.

First, we rewrite the local function $h_i$ in form given in Assumption~\ref{a: SA.update.function}. We define vector $b_i(Q,s,a,s')$ and matrix $A_i(Q,s,a,s'),$ for every $Q\in \bR^{\SA}$ and $(s,a,s')\in \cY,$ as 
\begin{align*}
    b_i(s,a,s') &:= \cR_i(s,a)\ones_{(s,a)} 
    \\
    A_i(Q,s,a,s') &:= \bI - \ones_{(s,a)}\big[ \ones_{(s,a)} - \gamma \ones_{(s', \pQ(s'))}\big]^\top.
\end{align*}
Then, $h_i,$ as defined in \eqref{appendix: local.function.exp}, takes the desired form 
\[
    h_i(Q,s,a,s') = b_i(s,a,s') + A_i(Q,s,a,s')Q.
\]
Now we are ready to verify condition~\ref{a: SA.update.function}. Since the semi-norm $\p{\cdot}$ in this case is the norm $\|\cdot\|_\infty,$ (see Section~\ref{s: application.RL}) the space $E$ contains only the zero vector, i.e., $E=\{0\}.$ Thus, condition~\ref{a: SA.update.function}.(a) holds trivially. Next, since 
\[
    A(Q',s,a,s')Q = \big[Q  - Q(s,a) \ones_{(s,a)}\big] + \gamma Q(s',\pi_{Q'}(s')),
\]
condition~\ref{a: SA.update.function}.(b) again holds by the greedy nature of the policy $\pQ.$ Condition~\ref{a: SA.update.function}.(c) holds if we set $\Cb=\Rm$ and $\CA= 2+\gamma;$ this is true because
\begin{multline*}
    \|b_i(s,a,s')\|_\infty 
    \\
    = \|\cR_i(s,a)\ones_{(s,a)}\|_\infty 
    \overset{(a)}{\leq} \Rm\|\ones_{(s,a)}\|_\infty \leq \Rm,
\end{multline*}
and
\begin{align*}
    & \|A_i(Q,s,a,s')\|_\infty 
    \\
    & \overset{(b)}{\leq} \|\bI\|_\infty + \|\ones_{(s,a)}\ones^\top_{(s,a)}\|_\infty + \gamma \|\ones_{(s,a)}\ones_{(s', \pQ(s'))}^\top\|_\infty
    \\
    & \overset{(c}{\leq} 1 + 1 + \gamma,
\end{align*}
where $(a)$ holds by assuming bounded reward functions, $(b)$ holds by the triangle inequality, and $(c)$ holds since $[\ones_{(s,a)}\ones_{(s, a)}^\top]$ and $[\ones_{(s,a)}\ones_{(s', \pQ(s'))}^\top]$ are both row-stochastic matrices. Lastly, condition \ref{a: SA.update.function}.(d) follows from the same partition described in Appendix~\ref{appendix: proof.avgQ.converge}, and defining $\cstr$ as the $\|\cdot\|_\infty$-distance of $\QS$ from the boundary of its cone $\cC_{\bar{a}^*}$ (see Appendix~\ref{appendix: proof.avgQ.converge}). Then, whenever $\|Q-\QS\|<\cstr,$ we have $Q$ and $\QS$ in the same cone; and hence $\pQ=\pi_{\QS}.$

This concludes the proof of Theorem~\ref{thm: expQ.convergence}. \qed

\subsection{Empirical Performance}
\begin{figure*}[!t]
    \centering

    \includegraphics[width=\linewidth]{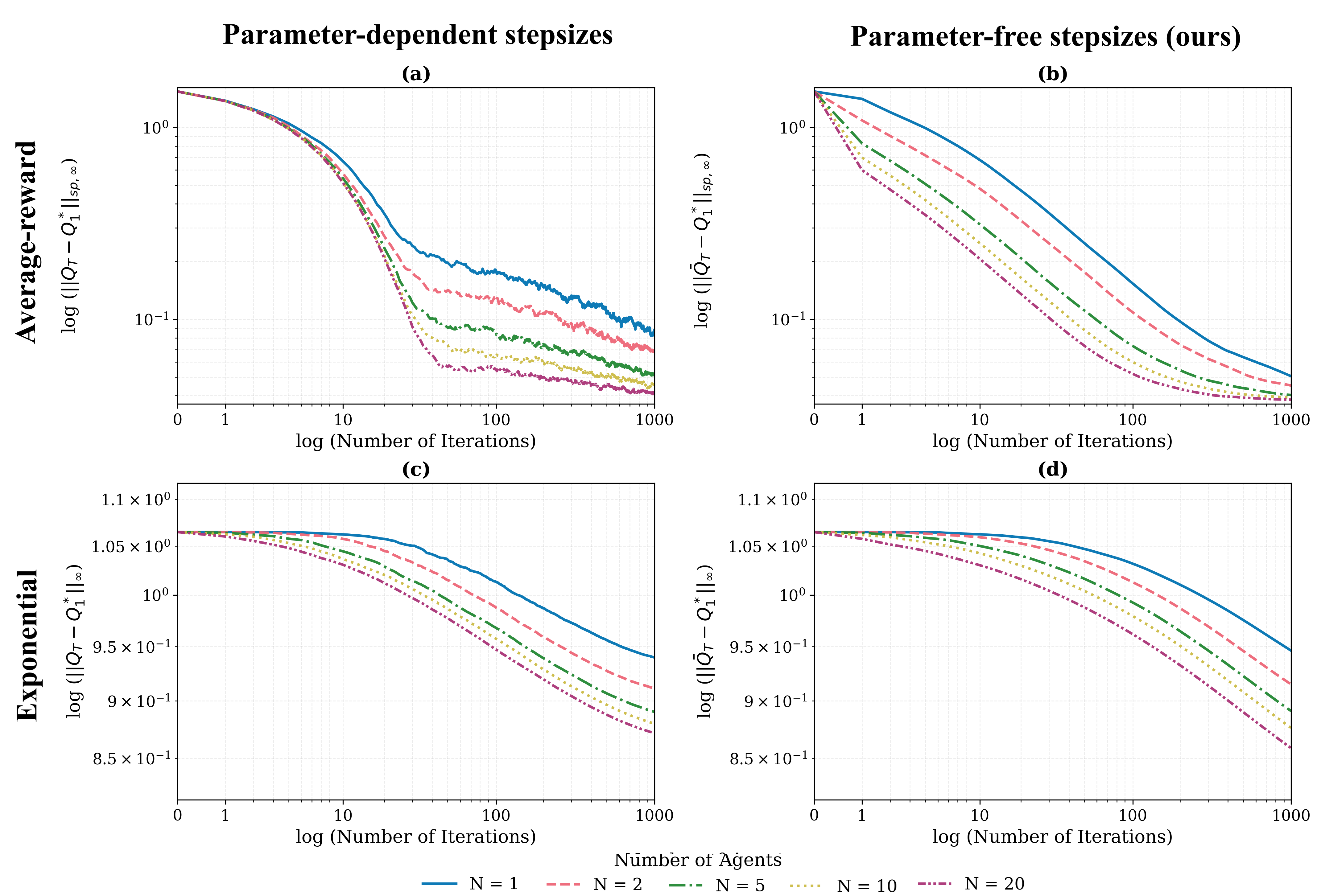}
    
    \caption{Performance of \textbf{distributed $Q$-learning} with $N$ agents. \textbf{Top row:} parameter-dependent (\textbf{left}) versus parameter-free (our work) (\textbf{right}) versions of synchronous distributed average-reward $Q$-learning. \textbf{Bottom row:} parameter-dependent and parameter-free versions of asynchronous distributed  exponentially-discounted $Q$-learning. Each curve shows the error averaged over $50$ independent runs of the algorithm, all initialized identically with $Q_0=\bQ_0=0$. The parameter-dependent versions use the linear stepsize $\alpha_t=c_1/t+c_2$, where $c_1=4/(1-\beta)$ and $c_2=\max\left\{\frac{2.88\log(\SA)}{(1-\beta)^2},3\right\}$, with $\beta=\gamma=0.1$ in the discounted setting and $\beta=0.75$ in the average-reward setting. The parameter-free stepsize is chosen as $\at=(t+1)^{-0.75}$. The error is measured with respect to the fixed point $\QS_1$ of a fixed MDP. The plots show that the iterates $(\bQ_T)$ converge to $\QS_1$ at the desired rate of $\tilde{O}(1/\sqrt{T})$. Moreover, performance improves as $N$ increases. Most importantly, our parameter-free algorithms (\textbf{right}) perform comparably to the parameter-dependent algorithms (\textbf{left}).
    \label{fig: FedQ}}
\end{figure*}

Figure~\ref{fig: FedQ} compares the performance of parameter-dependent and our proposed PR-averaging-based parameter-free versions of distributed average-reward and exponentially discounted $Q$-learning in synthetic settings. In both cases, the empirical behavior closely matches our theory: the error decays steadily with the number of iterations, and this decay becomes faster as the number of agents increases, highlighting the benefit of distribution. More importantly, the parameter-free method tracks the performance of the parameter-dependent baseline remarkably well, despite not relying on problem-specific tuning. These results suggest that the gains predicted by our analysis are not merely asymptotic, but are already visible in finite-time behavior across both the average-reward and discounted settings.

\section*{Acknowledgements}

This work was partially supported by ARO Grant W911NF2310111, Purdue University, and the DST-SERB Overseas Visiting Doctoral Fellowship. In addition, Gugan Thoppe’s research was supported in part by grants from the Walmart Centre for Tech Excellence at IISc, the Indo-French Centre for the Promotion of Advanced Research (CEFIPRA; Project 7102-1), the Kotak IISc AI-ML Centre FinTech Grant, the DST-SERB Core Research Grant (CRG/2021/008330), and the Pratiksha Trust Young Investigator Award. The authors thank Utsav Negi for assistance in generating Figure~\ref{fig: FedQ}.

% \section{\textbf{NOTE: To-do list}} 

% \begin{enumerate}
%     \item Change $O(1/T)$ to $\tilde{O}(1/T)$.
%     \item Add table of algorithms in application section: prioritize nonlinear updates and add examples beyond RL.

%     \item Add a result about the unaveraged iterates in the RL applications section. Write a story about how our work enables extending such a result to optimal parameter-free rates for the averaged iterates.

%     \item Keep Theorem 3.7 as a theorem, but change its conclusions to show that all the conditions under Assumptions~\ref{a: SA.noise} and \ref{a: SA.update.function} hold. 

%     \item Identify a list of non-linear stochastic approximation applications beyond $Q$-learning where our assumptions and this intermediate lemma hold.

%     \item Make a table showing the different application with explicit descriptions for $h_i, A_i,$ etc.
    
%     \item See if TD with Arbitrary Features can be subsumed.

%     \item Define $\tauT$ appropriately.
% \end{enumerate}

%%%%%%%%%%%%%%%%%%%%%%%%%%%%%%%%%%%%%%%%%%%%%%%%

\bibliography{cam_ready}

@article{liu2023distributed,
  title={Distributed TD (0) with almost no communication},
  author={Liu, Rui and Olshevsky, Alex},
  journal={IEEE Control Systems Letters},
  volume={7},
  pages={2892--2897},
  year={2023},
  publisher={IEEE}
}

@InProceedings{khodadadian22a,
  title = 	 {Federated Reinforcement Learning: Linear Speedup Under {M}arkovian Sampling},
  author =       {Khodadadian, Sajad and Sharma, Pranay and Joshi, Gauri and Maguluri, Siva Theja},
  booktitle = 	 {Proceedings of the 39th International Conference on Machine Learning},
  pages = 	 {10997--11057},
  year = 	 {2022},
  editor = 	 {Chaudhuri, Kamalika and Jegelka, Stefanie and Song, Le and Szepesvari, Csaba and Niu, Gang and Sabato, Sivan},
  volume = 	 {162},
  series = 	 {Proceedings of Machine Learning Research},
  month = 	 {17--23 Jul},
  publisher =    {PMLR},
  url = 	 {https://proceedings.mlr.press/v162/khodadadian22a.html}
}

@article{
wang2024federated,
title={Federated {TD} Learning with Linear Function Approximation under Environmental Heterogeneity},
author={Han Wang and Aritra Mitra and Hamed Hassani and George J. Pappas and James Anderson},
journal={Transactions on Machine Learning Research},
issn={2835-8856},
year={2024},
url={https://openreview.net/forum?id=hdQspgyFrk},
note={}
}

@article{dal2023federated,
  title={Federated td learning over finite-rate erasure channels: Linear speedup under markovian sampling},
  author={Dal Fabbro, Nicol{\`o} and Mitra, Aritra and Pappas, George J},
  journal={IEEE Control Systems Letters},
  volume={7},
  pages={2461--2466},
  year={2023},
  publisher={IEEE}
}

@misc{naskar2025,
      title={Parameter-free Optimal Rates for Nonlinear Semi-Norm Contractions with Applications to $Q$-Learning}, 
      author={Ankur Naskar and Gugan Thoppe and Vijay Gupta},
      year={2025},
      eprint={2508.05984},
      archivePrefix={arXiv},
      primaryClass={cs.LG},
      url={https://arxiv.org/abs/2508.05984}, 
}

@inproceedings{lakshminarayanan2018linear,
  title={Linear stochastic approximation: How far does constant step-size and iterate averaging go?},
  author={Lakshminarayanan, Chandrashekar and Szepesvari, Csaba},
  booktitle={International conference on artificial intelligence and statistics},
  pages={1347--1355},
  year={2018},
  organization={PMLR}
}

@article{dalal2018td0,
author = {Dalal, Gal and Szorenyi, Balazs and Thoppe, Gugan and Mannor, Shie},
year = {2018},
month = {04},
pages = {},
title = {Finite Sample Analyses for TD(0) With Function Approximation},
volume = {32},
journal = {Proceedings of the AAAI Conference on Artificial Intelligence},
doi = {10.1609/aaai.v32i1.12079}
}

@inproceedings{bhandari2018finite,
  title={A finite time analysis of temporal difference learning with linear function approximation},
  author={Bhandari, Jalaj and Russo, Daniel and Singal, Raghav},
  booktitle={Conference on learning theory},
  pages={1691--1692},
  year={2018},
  organization={PMLR}
}

@article{zhang2021finite,
  title={Finite Sample Analysis of Average-Reward TD Learning and $ Q $-Learning},
  author={Zhang, Sheng and Zhang, Zhe and Maguluri, Siva Theja},
  journal={Advances in Neural Information Processing Systems},
  volume={34},
  pages={1230--1242},
  year={2021}
}

@inproceedings{li2023statistical,
  title={A statistical analysis of {P}olyak-{R}uppert averaged {Q}-learning},
  author={Li, Xiang and Yang, Wenhao and Liang, Jiadong and Zhang, Zhihua and Jordan, Michael I},
  booktitle={International Conference on Artificial Intelligence and Statistics},
  pages={2207--2261},
  year={2023},
  organization={PMLR}
}

@article{chen2021lyapunov,
  title={A Lyapunov theory for finite-sample guarantees of asynchronous Q-learning and TD-learning variants},
  author={Chen, Zaiwei and Maguluri, Siva Theja and Shakkottai, Sanjay and Shanmugam, Karthikeyan},
  journal={arXiv preprint arXiv:2102.01567},
  year={2021}
}

@article{durmus2025finite,
  title={Finite-time high-probability bounds for Polyak--Ruppert averaged iterates of linear stochastic approximation},
  author={Durmus, Alain and Moulines, Eric and Naumov, Alexey and Samsonov, Sergey},
  journal={Mathematics of Operations Research},
  volume={50},
  number={2},
  pages={935--964},
  year={2025},
  publisher={INFORMS}
}

@article{even2003learning,
  title={Learning rates for Q-learning},
  author={Even-Dar, Eyal and Mansour, Yishay},
  journal={Journal of machine learning Research},
  volume={5},
  number={Dec},
  pages={1--25},
  year={2003}
}

@article{chen2021finite,
  title={Finite-sample analysis of off-policy td-learning via generalized bellman operators},
  author={Chen, Zaiwei and Maguluri, Siva Theja and Shakkottai, Sanjay and Shanmugam, Karthikeyan},
  journal={Advances in Neural Information Processing Systems},
  volume={34},
  pages={21440--21452},
  year={2021}
}

@Article{repec:inm:oropre:v:69:y:2021:i:3:p:950-973,
journal={Operations Research},
author={Jalaj Bhandari and Daniel Russo and Raghav Singal},
title={A Finite Time Analysis of Temporal Difference Learning with Linear Function Approximation},
year={2021},
month={May},
pages={950-973},
volume={69},
number={3},
doi={10.1287/opre.2020.2024},
url={https://ideas.repec.org/a/inm/oropre/v69y2021i3p950-973.html},
}

@article{sayin2021decentralized,
  title={Decentralized Q-learning in zero-sum Markov games},
  author={Sayin, Muhammed and Zhang, Kaiqing and Leslie, David and Basar, Tamer and Ozdaglar, Asuman},
  journal={Advances in Neural Information Processing Systems},
  volume={34},
  pages={18320--18334},
  year={2021}
}

@article{hu2003nash,
  title={Nash Q-learning for general-sum stochastic games},
  author={Hu, Junling and Wellman, Michael P},
  journal={Journal of machine learning research},
  volume={4},
  number={Nov},
  pages={1039--1069},
  year={2003}
}

@inproceedings{li2018algorithmic,
  title={Algorithmic regularization in over-parameterized matrix sensing and neural networks with quadratic activations},
  author={Li, Yuanzhi and Ma, Tengyu and Zhang, Hongyang},
  booktitle={Conference On Learning Theory},
  pages={2--47},
  year={2018},
  organization={PMLR}
}

@article{marjanovic2012l_q,
  title={On $ l\_q $ optimization and matrix completion},
  author={Marjanovic, Goran and Solo, Victor},
  journal={IEEE Transactions on signal processing},
  volume={60},
  number={11},
  pages={5714--5724},
  year={2012},
  publisher={IEEE}
}

@inproceedings{lohmiller1997applications,
  title={Applications of contraction analysis},
  author={Lohmiller, Winfried and Slotine, J-JE},
  booktitle={Proceedings of the 36th IEEE Conference on Decision and Control},
  volume={2},
  pages={1044--1049},
  year={1997},
  organization={IEEE}
}

@article{wainwright2019stochastic,
  title={Stochastic approximation with cone-contractive operators: Sharp $ \ell_\infty $-bounds for $ Q $-learning},
  author={Wainwright, Martin J},
  journal={arXiv preprint arXiv:1905.06265},
  year={2019}
}

@inproceedings{manchester2014output,
  title={Output-feedback control of nonlinear systems using control contraction metrics and convex optimization},
  author={Manchester, Ian R and Slotine, Jean-Jacques E},
  booktitle={2014 4th Australian control conference (AUCC)},
  pages={215--220},
  year={2014},
  organization={IEEE}
}

@incollection{manchester2017unifying,
  title={Unifying robot trajectory tracking with control contraction metrics},
  author={Manchester, Ian R and Tang, Justin Z and Slotine, Jean-Jacques E},
  booktitle={Robotics Research: Volume 2},
  pages={403--418},
  year={2017},
  publisher={Springer}
}

@article{tsukamoto2021contraction,
  title={Contraction theory for nonlinear stability analysis and learning-based control: A tutorial overview},
  author={Tsukamoto, Hiroyasu and Chung, Soon-Jo and Slotine, Jean-Jaques E},
  journal={Annual Reviews in Control},
  volume={52},
  pages={135--169},
  year={2021},
  publisher={Elsevier}
}

@inproceedings{patil2023finite,
  title={Finite time analysis of temporal difference learning with linear function approximation: Tail averaging and regularisation},
  author={Patil, Gandharv and Prashanth, LA and Nagaraj, Dheeraj and Precup, Doina},
  booktitle={International Conference on Artificial Intelligence and Statistics},
  pages={5438--5448},
  year={2023},
  organization={PMLR}
}

@inproceedings{naskar2024federated,
  title={Federated TD Learning in Heterogeneous Environments with Average Rewards: A Two-timescale Approach with Polyak-Ruppert Averaging},
  author={Naskar, Ankur and Thoppe, Gugan and Koochakzadeh, Abbasali and Gupta, Vijay},
  booktitle={2024 IEEE 63rd Conference on Decision and Control (CDC)},
  pages={387--393},
  year={2024},
  organization={IEEE}
}

@article{chen2025non,
  title={A non-asymptotic theory of seminorm lyapunov stability: From deterministic to stochastic iterative algorithms},
  author={Chen, Zaiwei and Zhang, Sheng and Zhang, Zhe and Haque, Shaan Ul and Maguluri, Siva Theja},
  journal={arXiv preprint arXiv:2502.14208},
  year={2025}
}

@book{ruppert1991stochastic,
    author = {D. Ruppert},
    title = {Stochastic approximation},
    publisher = {In Handbook of Sequetial Analysis},
    pages = {503-529},
    year = {1991}
}

@article{polyak1992acceleration,
  title={Acceleration of stochastic approximation by averaging},
  author={Polyak, Boris T and Juditsky, Anatoli B},
  journal={SIAM journal on control and optimization},
  volume={30},
  number={4},
  pages={838--855},
  year={1992},
  publisher={SIAM}
}

@book{borkar2009,
  title={Stochastic approximation: a dynamical systems viewpoint},
  author={Borkar, Vivek S},
  volume={48},
  year={2009},
  publisher={Springer}
}

@book{sutton2018reinforcement,
  title={Reinforcement Learning: An Introduction},
  author={Sutton, Richard S and Barto, Andrew G},
  year={2018},
  publisher={MIT Press}
}

@book{szepesvari2022algorithms,
  title={Algorithms for reinforcement learning},
  author={Szepesv{\'a}ri, Csaba},
  year={2022},
  publisher={Springer nature}
}

@book{BertsekasTsitsiklis1996,
  author    = {Bertsekas, Dimitri P. and Tsitsiklis, John N.},
  title     = {Neuro-Dynamic Programming},
  publisher = {Athena Scientific},
  year      = {1996},
  address   = {Belmont, MA}
}

@misc{naskar2025federated,
      title={Parameter-Free Federated TD Learning with Markov Noise in Heterogeneous Environments}, 
      author={Ankur Naskar and Gugan Thoppe and Utsav Negi and Vijay Gupta},
      year={2025},
      eprint={2510.07436},
      archivePrefix={arXiv},
      primaryClass={cs.LG},
      url={https://arxiv.org/abs/2510.07436}, 
}

\end{document}